\documentclass{article} 
\usepackage{iclr2025_conference,times}

\usepackage{amsmath,amsfonts,bm}

\def\eqref#1{equation~\ref{#1}}

\def\1{\bm{1}}

\DeclareMathAlphabet{\mathsfit}{\encodingdefault}{\sfdefault}{m}{sl}
\SetMathAlphabet{\mathsfit}{bold}{\encodingdefault}{\sfdefault}{bx}{n}

\newcommand{\E}{\mathbb{E}}

\newcommand{\R}{\mathbb{R}}

\newcommand{\KL}{D_{\mathrm{KL}}}

\usepackage{amsthm, amsfonts,amssymb,amsmath}
\usepackage{hyperref}
\usepackage{url}
\usepackage{wrapfig}
\usepackage{algorithm}
\usepackage{algpseudocode}
\usepackage{comment}
\usepackage{pifont}
\usepackage{mathrsfs}  
\usepackage{subcaption}
\usepackage{siunitx}
\usepackage{graphicx}
\usepackage{caption}
\usepackage{mathtools}
\mathtoolsset{showonlyrefs}
\usepackage [autostyle, english = american]{csquotes}
\MakeOuterQuote{"}

\theoremstyle{definition}
\newtheorem{assumption}{Assumption}
\newtheorem{definition}{Definition}
\theoremstyle{plain}
\newtheorem{theorem}{Theorem}
\newtheorem*{theorem*}{Theorem}
\title{Provable Accuracy Bounds for \\Hybrid Dynamical Optimization and Sampling}

\author{Matthew X. Burns, Qingyuan Hou \& Michael C. Huang \\
Department of Electrical and Computer Engineering\\
University of Rochester\\
Rochester, NY 14527, USA \\
\texttt{\{mburns13,qhou3\}@ur.rochester.edu} \\
\texttt{michael.huang@rochester.edu} \\
}

\newcommand{\diverg}{\nabla\cdot}
\newcommand{\brackets}[1]{\left[#1\right]}
\newcommand{\inner}[2]{\left\langle#1,#2\right\rangle}
\newtheorem{lemma}{Lemma}
\newtheorem{corollary}{Corollary}
\renewcommand{\KL}[2]{\operatorname{D_{\mathrm{KL}}}({#1}\Vert{#2})}
\newcommand{\kl}{\operatorname{D}_{\mathrm{KL}} }

\newcommand{\norm}[1]{\lVert #1 \rVert }
\iclrfinalcopy 
\begin{document}

\maketitle
\begin{abstract} Analog dynamical accelerators (DXs) are a growing sub-field in computer architecture research, offering order-of-magnitude gains in power efficiency and latency over traditional digital methods in several machine learning, optimization, and sampling tasks. However, limited-capacity accelerators require hybrid analog/digital algorithms to solve real-world problems, commonly using large-neighborhood local search (LNLS) frameworks. Unlike fully digital algorithms, hybrid LNLS has no non-asymptotic convergence guarantees and no principled hyperparameter selection schemes, particularly limiting cross-device training and inference.

In this work, we provide non-asymptotic convergence guarantees for hybrid LNLS by reducing to block Langevin Diffusion (BLD) algorithms.
Adapting tools from classical sampling theory, we prove exponential KL-divergence convergence for randomized and cyclic block selection strategies using ideal DXs. With finite device variation, we provide explicit bounds on the 2-Wasserstein bias in terms of step duration, noise strength, and function parameters. Our BLD model provides a key link between established theory and novel computing platforms, and our theoretical results provide a closed-form expression linking device variation, algorithm hyperparameters, and performance.
\end{abstract}
\section{Introduction}\label{sec:intro}
Computing research has long borrowed from the physical sciences. Sampling and optimization algorithms such as simulated annealing~\citep{kirkpatrick_optimization_1984}, parallel tempering~\citep{jearl_parallel_2005}, and  Langevin Monte Carlo (LMC)~\citep{chewi_analysis_2021} were directly inspired by physical processes observed in nature. In like spirit, a growing computer architecture sub-field has proposed leveraging physical dynamics to accelerate computationally expensive workloads using ``dynamical accelerators'' (DXs). Originally, research focused on combinatorial optimization problems~\citep{inagaki_coherent_2016,ushijima-mwesigwa_graph_2017,wang_oim_2019,afoakwa_brim_2021,  mohseni_ising_2022} and matrix-vector multiplication~\cite{xiao_accuracy_2022}. However, the field has expanded to sampling for energy-based model training and inference~\citep{vengalam_supporting_2023} and generative inference in graph neural networks~\citep{wuextending,song_ds-gl_2024}. 

The interest in analog acceleration coincides with novel proposals for ``local update'' algorithms, where layer activations $h$ are solutions to a minimization problem $h^*_\ell=\textrm{argmin}_{h} f(h)$~\citep{scellier_equilibrium_2017,stern_supervised_2021,millidge_predictive_2022,scellier_energy-based_2023}. While costly in digital systems, stochastic analog optimizers can effectively solve $\textrm{argmin}_{h} f(h)$ in minimal time and energy~\citep{wuextending}, making them suitable candidates for local-update learning implementations.

However, real-world problems are typically too large for dynamical accelerators to optimize in their entirety, requiring routines to partition and iteratively sample/optimize subspaces~\citep{booth_partitioning_nodate,sharma_increasing_2022,song_ds-gl_2024}, most commonly using hybrid ``large-neighborhood local search'' (LNLS) frameworks~\citep[][]{ahuja_survey_2002,booth_partitioning_nodate}. In hybrid LNLS, the DX is used to perform alternating sampling/minimization over within-capacity subproblems. However, hybrid LNLS has undergone little theoretical examination. No non-asymptotic convergence bounds yet exist, limiting the appeal of hybrid LNLS compared with more well-understood digital algorithms. Moreover, the effect of algorithm hyperparameters on convergence and their interplay with device non-idealities is unclear. Models trained on one DX may require hyperparameter adjustment, if not outright device-specific retraining, prior to inference on another~\citep{he_noise_2019,long_design_2019}.  Without non-asymptotic analysis linking device variation and accelerator convergence, accelerator adaptation reduces to trial-and-error.

In this work, we provide the first explicit probabilistic convergence guarantees for hybrid LNLS algorithms in activation sampling and optimization: a crucial first step in optimizing and analyzing hybrid DX frameworks. 
We start by reducing hybrid LNLS to block sampling with continuous-time, Langevin diffusion-based sub-samplers, to which we can apply tools from classical sampling analysis. Two block selection rules for ``block Langevin diffusion'' (BLD) are examined, randomized and cyclic, using ideal (Secs.~\ref{sec:rbld} and ~\ref{sec:cbld}) and finite-variation (Sec~\ref{sec:variation}) analog components. Under a log-Sobolev inequality (LSI), we prove that ideal accelerators converge to the target distribution exponentially fast. However, we show that finite device variation incurs a bias in $W_2$ distance, proportional to step duration and dependent on variation magnitude. We illustrate our findings with numerical experiments on a toy Gaussian sampling problem, demonstrating the effect of device variation and hyperparameter choice on convergence. 

Our contributions can be summarized as follows: 
\begin{enumerate}
    \item We provide the first bounds on randomized block diffusions using explicit constants (Theorem~\ref{thm:rbld_kl}), strengthening the results of~\cite{ding_random_2020}
    \item We provide completely novel bounds for cyclic block diffusions (Theorem~\ref{theorem:CBLD_convergence}) by proving a general conditional sampling lemma for Kullback-Liebler divergence (Lemma~\ref{lemma:cyclic_kl_general})
    \item Using a Talagrand transportation inequality, we combine our ideal results with analysis following ~\cite{raginsky_non-convex_2017} to provide non-asymptotic guarantees  for DXs with analog non-idealities (Theorem ~\ref{theorem:w2_conv}), applicable to both sampling and optimization tasks.
\end{enumerate}

\section{Background}
\subsection{Related Works}
\cite{ding_langevin_2021} and \cite{ding_random_2021} proposed and analyzed ``randomized coordinate Langevin Monte Carlo'' (RCLMC) methods for sampling tasks using over and underdamped Langevin dynamics. Their methodology used Wasserstein coupling arguments akin to ~\cite{dalalyan_theoretical_2016}, in contrast to our interpolation arguments following~\cite{vempala_rapid_2019}. Accordingly, the authors assumed a strongly-log concave target distribution: a much stronger assumption than an LSI. Moreover, \cite{ding_random_2021} provided insufficient analysis for the continuous-time case, focusing primarily on the discrete RCLMC algorithm. DX algorithm analysis required continuous-time bounds with explicit constants, necessitating our contributions.

Two algorithms related to BLD garnering recent interest are ``coordinate ascent variational inference'' (CAVI), which performs variational inference over factorized ``mean-field'' distributions~\citep{bhattacharya_convergence_2023,arnese_convergence_2024}, and the split Gibbs sampler (SGS), which alternates sampling over problem variables with augmented priors~\citep{vono_split-and-augmented_2019,vono_efficient_2022}. CAVI is similar to BLD, and indeed the information theoretic analysis by~\cite{lee_gibbs_2022} has a similar structure to our proof of Lemma~\ref{lemma:cyclic_kl_general}. SGS has been likened to ADMM in optimization~\cite{vono_efficient_2022}, indicating there may be an equivalence to BLD akin to classical block optimization~\cite{tibshirani_dykstra_2017}.

A related class of works have analyzed the accuracy of analog matrix-vector multiplication (MVM) accelerators in neural network inference~\citep{klachko_improving_2019,xiao_accuracy_2022}. MVM accelerators are a restricted class of DXs minimizing $\min_{y\in\R^{d}}||y-Wx||^2$: equivalent to performing MVM in the analog domain. Our analysis generalizes MVM analysis and is applicable in more complex analog settings such as generative sampling~\citep{vengalam_supporting_2023,melanson_thermodynamic_2023,wuextending}. 

Optimization-based convergence analyses of specific DX architectures were carried out by~\cite{erementchouk_computational_2022,pramanik_convergence_2023}. Asymptotic convergence in expectation to the global minimizer was proved by~\cite{pramanik_convergence_2023} in the zero-temperature limit with decreasing stepsize, echoing our results in Sec.~\ref{sec:variation}. However, neither work accounted for the effect of device variation or problem partitioning, and both focused on specific DX modalities (nonlinear electronic/optical oscillators) rather than a general model of DX behavior. Information-theoretic analysis conducted by~\cite{dambre_information_2012,hu_tackling_2023} have bounded the asymptotic computational capabilities of DX systems, but not their probabilistic convergence.
\subsection{Dynamical Accelerators}
The first wave of dynamics-accelerated optimizers primarily targeted the Ising Spin Glass (ISG) Hamiltonian from statistical physics, earning the appellation ``Ising Machines''. The ISG Hamiltonian describes quadratic interactions between binary ``spins'', which can be used to solve intractable combinatorial problems~\citep{lucas_ising_2014}. Ising machines have been implemented using quantum spins~\citep{ushijima-mwesigwa_graph_2017}, electronic~\citep{wang_oim_2019,albertsson_ising_2023} and optical~\citep{inagaki_coherent_2016,honjo_100000-spin_2021} oscillator phases, resistively-coupled capacitors~\citep{afoakwa_brim_2021}, and many more besides~\citep{mohseni_ising_2022}. These initial prototypes successfully optimized binary target functions. However, recent architectures have broader applications domains: with support for non-quadratic cost functions~\citep{sharma_combining_2023,bashar_designing_2023,bybee_efficient_2023} and continuous values~\citep{brown_accelerating_2024,wuextending,song_ds-gl_2024}. Since these designs have moved beyond the ISG Hamiltonian, we term this broader class simply as ``dynamical accelerators'' (DXs).

While the physical implementation differs between DXs, several proposals 
can be described by a Langevin stochastic differential equation (SDE)
\begin{equation}\label{eqn:sde_generic}
    dx_t=-\nabla h(x_t, t)dt +\sqrt{2\beta(t) ^{-1}}dW_t
\end{equation}
where $x_t\triangleq x(t)$ is the system state, $dW_t$ is a Brownian noise term, $h(x, t)$ is the deterministic system potential, and $\beta(t)=1/T(t)$ is the (potentially time dependent) inverse pseudo-temperature of the system. 

$x(t)$ represents the continuous, physical degrees of freedom of the optimizer/sampler, such as capacitor voltage~\citep{afoakwa_brim_2021} or oscillator phase~\citep{inagaki_coherent_2016,wang_oim_2019}. Several DX prototypes have been shown to follow forms of Equation~\eqref{eqn:sde_generic}, either intentionally to escape local minima~\citep{wang_oim_2019,sharma_combining_2023,aifer_thermodynamic_2023} or unintentionally to model dynamic environment noise~\citep{wang_coherent_2013}. The potential $h(x,t)$ includes the target function $f(x)$ along with optional time-dependent terms, such as a sub-harmonic injection locking potential for binarization~\citep{wang_oim_2019}. 

DXs are also prone to static ``device variation'' owing to analog non-idealities. Unlike the Brownian term $dW_t$, static non-idealities are not self-averaging, and result in a biased estimate $g_\delta(x)$ of the gradient $\nabla f(x)$. In a quadratic function $f(x)=x^TWx$, for instance, the gradient estimate can be described as $g_\delta(x)= (W+W^T)x+\delta x$, where $\delta_{ij}\sim \mathcal{N}(0,\Delta^2)$ are fixed non-idealities in device components. Previous studies have examined the impact of static variation on binary optimization~\citep{albash_analog_2019} and matrix-vector multiplication~\citep{xiao_accuracy_2022}, but have not extended to non-asymptotic convergence analysis for more general functions over $\R^d$.

\subsection{Langevin Diffusion}
If we restrict our analysis to the time-homogeneous case where $h(x,t)=f(x)$, $\beta(t)=\beta$, the dynamics are Markovian with a constant stationary distribution
\begin{equation}
    \pi_\beta(x)\propto e^{-\beta f(x)}.
\end{equation}
The Langevin SDE
\begin{equation}
    dx_t=-\nabla f(x_t)dt +\sqrt{2\beta ^{-1}}dW_t
\end{equation}
produces a continuous sample path $x(t)$ with each $x(\tau), \tau\geq 0$ acting as a random variable. The law of $x_t$, $\mu_t$ (denoted $\mu_t=\mathcal{L}(x_t)$), is described by the \emph{Fokker-Planck} equation (FPE)
\begin{equation}\label{eqn:fpe}
    \partial_t\mu_t=\beta^{-1} \nabla^2\cdot\mu_t + \nabla\cdot{[\mu_t \nabla f(x_t)]}.
\end{equation}

The Langevin SDE describes the physical evolution of $x_t$, while the FPE describes the change in the sample distribution $\mu_t$ in measure space. If $\pi_\beta$ satisfies a log-Sobolev inequality (LSI, see Sec.~\ref{sec:results}), then $\mu_t$ converges to $\pi_\beta$ exponentially fast~\citep[Theorem 1][]{vempala_rapid_2019}.
\setcounter{theorem}{-1}
\begin{theorem}[LD Convergence (Theorem 1 of \cite{vempala_rapid_2019})]\label{thm:ld_conv}
Suppose $\pi_\beta(x) \propto e^{-\beta f(x)}$ satisfies an LSI with constant $1/\gamma$. Then the distribution $\mu_t$ of the Langevin diffusion at time $t$ satisfies 
\begin{equation}\label{eqn:ld_kl_conv}
    \KL{\mu_t}{\pi} \leq e^{-2\gamma\beta^{-1} t}\KL{\mu_0}{\pi}
\end{equation}
\end{theorem}
where $\KL{\mu_t}{\pi}\triangleq \int\mu_t(x)\log\frac{\mu_t(x)}{\pi(x)}dx\triangleq\int\mu_t\log\frac{d\mu_t}{d\pi}$ is the Kullback-Leibler (KL) divergence between two probability measures. 

Recalling the Otto-Villani theorem~\citep[Theorem 1][]{otto_generalization_2000}, an LSI inequality further implies a Talagrand transportation inequality
\begin{equation}
    W_2(\mu_t,\pi) \leq \left(\frac{2}{\gamma}\right)^{1/2} \sqrt{\KL{\mu_t}{\pi}}
\end{equation}
where $W_2(\mu_t,\pi)=\inf_{\nu\in \mathcal{C}(\mu_t, \pi)}\left(\int ||x-y||^2_2\nu (x,y)dxdy\right)^{1/2}$ is the 2-Wasserstein distance between $\mu_t$ and $\pi$ and $\nu\in \mathcal{C}(\mu_t, \pi)$ is a \emph{coupling} over $\mu_t,\:\pi$. Convergence in $\kl$ under an LSI therefore implies convergence in $W_2$, allowing us to state bounds in both. Crucially for our purposes, the 2-Wasserstein distance is a metric over probability distributions, allowing use of the triangle inequality~\citep{raginsky_non-convex_2017}.

As $\beta\to \infty$, $\pi_\beta(x)$ concentrates around the minimizer(s) of $f$. This observation permits us to unite optimization and sampling using annealing schemes~\citep{kirkpatrick_optimization_1984,chiang_diffusion_1987,chak_generalized_2023} which gradually increase $\beta$ to escape early local minima and (hopefully) find the global minimum, indicating a direction for future work extending our results. 
Previous works have also used bounds on convergence to $\pi_\beta$ at constant $\beta$ to bound optimizer hitting times~\cite{zhang_hitting_2017} and expected excess risk~\citep{raginsky_non-convex_2017,xu_global_2020,farghly_time-independent_2021, zhang_nonasymptotic_2023} in non-convex optimization.
\section{Main Results}\label{sec:results}
\newcommand{\myalgorithm}{%
\begin{algorithm}[H]
\SetAlgoLined
\KwData{this text}
\KwResult{how to write algorithm with \LaTeX2e }
initialization\;
\While{not at end of this document}{
read current\;
\eIf{understand}{
go to next section\;
current section becomes this one\;
}{
go back to the beginning of current section\;
}
}
\end{algorithm}}
\subsection{LNLS as Block Sampling}
        

DXs have a finite capacity, necessitating hybrid analog/digital algorithms to solve problems exceeding that capacity. A popular candidate for hybrid optimization/sampling is the Large-Neighborhood Local Search (LNLS) framework~\citep{ahuja_survey_2002,booth_partitioning_nodate,sharma_increasing_2022,raymond_hybrid_2023}, where a local solver (the DX) is used to optimize/sample blocks of variables $\{B_1,B_2,...,B_b\}$ conditioned on the rest of the problem state, illustrated in Fig.~\ref{fig:block_diffusion_illustration}. 

\begin{figure}[ht]
    \centering
    \includegraphics[width=0.7\linewidth]{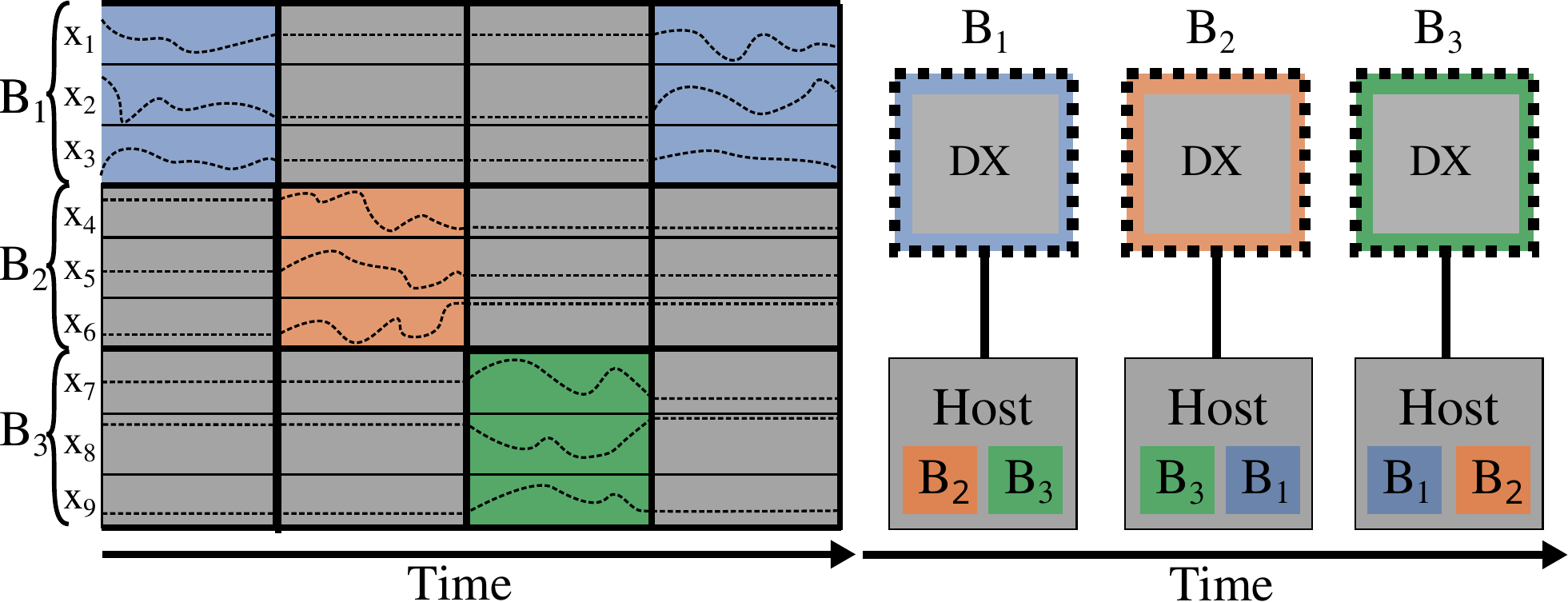}
    \caption{Illustration of the LNLS algorithm on a 3-block, 9-variable problem. \textbf{[Left]} An illustration of the variable sample paths during algorithm execution. When a block is not being actively evolved, the constituent variables remain fixed (gray). \textbf{[Right]} Logical partition of variables in an LNLS framework, where one block is being actively evolved by the DX with the others resident in digital memory. The digital host performs the control operations needed to read the block state, write back to memory, and begin the next block evolution.}
    \label{fig:block_diffusion_illustration}
\end{figure}

We can formalize LNLS by borrowing notation from classical coordinate descent~\citep{nesterov_efficiency_2012,beck_convergence_2013}. We decompose the space $\R^d$ into disjoint blocks $\bigtimes_{i=1}^b B_i$, where each subspace $B_i$ has dimension $d_i$. LNLS frameworks perform alternating sampling from each conditional distribution $\mu_{B_i|B_1\ldots B_b}$, where each block is chosen at random or in a deterministic order.

For simplified notation, we decompose $I^d=\sum_{i=1}^b U_i$ where each $U_i\in\R^{d\times d}$ has ones along diagonal indices corresponding to dimensions in $B_i$ and zeros elsewhere. Then $\sum_{i=1}^bU_i\nabla f(x)=\nabla f(x)$ and we can express the SDE for a single block $B_i$ diffusion as
\begin{equation}\label{eqn:sde_block}
    dx=-U_i\nabla f(x)dt+U_i\sqrt{2\beta^{-1}}dW_{t}.
\end{equation}
Equation~\eqref{eqn:sde_block} leaves the conditioned dimensions $\overline B_{i}\triangleq\{j\in\{1,...,d\}:j\not\in B_i\}$ invariant. Each block diffusion occurs in continuous time, but the blocks are swapped at discrete steps. Accordingly, we denote $x(t_k)$ as the iterate at time $t$ in block step $k$ and $\mu_{t_k}$ as its associated probability distribution.
\begin{algorithm}[t]
\caption{Block Langevin Diffusion (BLD)}
\label{alg:bld}
    \begin{algorithmic}[1]
        \Procedure{BLD}{$x^0\in\textrm{dom}(f)$, Decomposition $\{B_1,...,B_b\}$, Step Size Set $\lambda\in\mathbb{R}^b_+$}
        \For{$k\geq 0$}
            \State Choose block index $i$ (Random or Deterministic)
            \State Set $t_{k+1}=t_k + \lambda_i$
            \State Sample:
            \vspace{-10px}
            \begin{align}
                x(t_{k+1})&=x(t_k)-\int_{t_k}^{t_{k+1}}U_i\nabla f(x)dt +\int_{t_k}^{t_{k+1}}\:U_i\sqrt{2\beta^{-1}}dW_t
            \end{align}
            \vspace{-10px}
        \EndFor
        \EndProcedure
    \end{algorithmic}
\end{algorithm}

When each block evolves at constant $\beta$ according to Equation~\eqref{eqn:sde_block}, we can model LNLS as a block sampling algorithm, termed \emph{Block Langevin Diffusion} (BLD), shown in Algorithm~\ref{alg:bld}. BLD is a continuous-time generalization of ``randomized coordinate Langevin Monte Carlo'' (RCLMC) studied in \cite{ding_random_2021,ding_langevin_2021}. By reducing LNLS to BLD, we can tractably analyze algorithm performance using well-developed tools from stochastic process analysis. As expressed, Algorithm~\ref{alg:bld} leaves open the choice of block selection. Here we consider randomized and cyclic selection rules, denoted \emph{Randomized Block Langevin Diffusion} (RBLD) and \emph{Cyclic Block Langevin Diffusion} (CBLD) respectively.

Throughout our analysis, we make the following assumptions on $f$.
\begin{assumption}\label{assmp:diff}
    $f$ is continuously differentiable
\end{assumption}
\begin{assumption}\label{assmp:lsi}
    $\pi_\beta\propto\exp[-\beta f(x)]$ satisfies a log-Sobolev inequality (LSI) with $C_{\textrm{LSI}}=\frac{1}{\gamma}$. That is, for all distributions $\mu$ with finite second moment
    \begin{equation}
        \KL{\mu}{\pi_\beta}\triangleq \int\mu(x)\log\frac{\mu(x)}{\pi_\beta(x)}dx\leq\frac{1}{2\gamma}\overbrace{\int\mu(x)\Vert \nabla \log\frac{\mu(x)}{\pi_\beta(x)}\Vert ^2dx}^{\operatorname{FI}(\mu\Vert \pi_\beta)}
    \end{equation}
\end{assumption}
where $\operatorname{FI}(\mu\Vert \pi_{\beta})$ is the (relative) \textit{Fisher information}. 
An LSI can hold even in non-log-concave distributions, making it a more general assumption than the strong log-concavity presumed by~\cite{ding_langevin_2021}. Examples include globally strongly log-concave measures with bounded regions of non-log concavity~\citep{raginsky_non-convex_2017,ma_sampling_2019}, high-temperature spin systems~\citep{bauerschmidt_very_2019} and heavy-tailed distributions which may not be \emph{strongly} log-concave (such as measures with potentials $f$ satisfying a Kurdyka-Łojasiewicz inequality~\cite{bolte2010characterizations}).

\subsection{Randomized Block Langevin Diffusion}\label{sec:rbld}
In the randomized case, we select the next variable block according to a probability mass function $\phi\in\R^b$ with all $\phi_i>0$. ~\cite{ding_random_2021} analyzed RCLMC using Wasserstein coupling arguments, however our analysis builds on the traditional proof of Equation~\eqref{eqn:ld_kl_conv} which relies on the \emph{de Bruijin identity}
\begin{equation}
    \partial_t\KL{\mu_t}{\pi_\beta}=-\beta^{-1}\operatorname{FI}(\mu_t\Vert \nu)
\end{equation}
which, when combined with the LSI, proves exponential convergence since $-\operatorname{FI}(\mu_t\Vert \nu)\leq -2\gamma\KL{\mu_t}{\pi_\beta}$. In the same vein, we use probabilistic arguments in Appendix~\ref{ssec:fpe_rcld} to prove a de Bruijin \emph{in}equality
\begin{equation}
    \partial_t\KL{\mu_t}{\pi_\beta}\leq-\phi_{\min}\beta^{-1}\operatorname{FI}(\mu_t\Vert \nu)
\end{equation}
where $\phi_{\min}$ is the minimum block probability in $\phi$.

By integrating and expanding the inequality, we easily obtain convergence in $\KL{\mu(t_k)}{\pi_\beta}$, expressed in Theorem~\ref{thm:rbld_kl}. We also prove convergence for a discrete-time variant (RBLMC) in Appendix~\ref{sec:rbld_proof}.

\begin{theorem}[RBLD $\KL{\mu_{t_k}}{\pi_\beta}$ Convergence]\label{thm:rbld_kl}
    Let $\phi=(\phi_1,...,\phi_b)$ be the block selection probability mass function, $\phi_i>0$, and let $\lambda=(\lambda_1,...,\lambda_b)\in\R^b$ be the sampling times for each block, $\lambda_i>0$. For any $\pi_\beta\propto\exp[-\beta f(x)]$ satisfying Assumptions~\ref{assmp:diff} and~\ref{assmp:lsi}, and any $\beta>0$, the sample distribution after $k$ steps of RBLD ($\mu_{t_k}$) satisfies
    \begin{equation}\label{eqn:rbld_bound}
        \KL{\mu_{t_k}}{\pi_\beta}\leq e^{-2\gamma\beta^{-1}\phi_{\min}\lambda_{\min}k}\KL{\mu_0}{\pi_\beta}.
    \end{equation}
\end{theorem}

\subsection{Cyclic Block Langevin Diffusion}\label{sec:cbld}
While our randomized results tighten existing theory, real-world instances of LNLS often use cyclic orderings~\cite{sharma_increasing_2022,song_ds-gl_2024,wuextending}, as they are more amenable to direct hardware and software optimization and are simpler to implement. 

However, unlike RBLD, we cannot easily prove a ``de Bruijin inequality'' for CBLD. Instead, we make extensive use of the \emph{chain lemma} for $\kl$
\begin{equation}
    \KL{\mu}{\nu}=\E_{\mu_D}[\KL{\mu_{A|E}}{\nu_{A|E}}]+\KL{\mu_{D}}{\nu_{E}}.
\end{equation}
where $A,E$ are disjoint subspaces of $\R^d$, $A\cup E=\R^d$, and $\mu$, $\nu$ are measures supported on $\R^d$ with $\mu_{A|E}$ denotes the measure over $A$ conditioned on $E=y$ for arbitrary $y$. Note that if we set $A=B_i$, $E=\overline B_i$, the CBLD diffusion will result in exponential contraction in $\E_{\mu_{\overline{B}_i}}[\KL{\mu_{B_i|\overline B_i}}{\nu_{B_i|\overline B_i}}]$ while leaving $\KL{\mu_{\overline B_i}}{\nu_{\overline B_i}}$ constant. CBLD then trivially results in non-increasing  $\KL{\mu(t_k)}{\pi_\beta}$, however expressing descent across iterations is more subtle due to the sub-additivity of KL-divergence. 

Taking inspiration from ~\cite{beck_convergence_2013}, we bound descent across $b$ steps, an entire ``cycle'' over the problem space, expressed in a general lemma for $\KL{\mu_{t_k}}{\pi_\beta}$ (proved in Appendix~\ref{sec:cyclic_appdx}).
\begin{lemma}[Cyclic KL Contraction]\label{lemma:cyclic_kl_general}
    Let the set $C=\{C_1,...,C_b\}\in\R_+^b$ satisfy $0< C_i< 1$, and let $D_i\in \R$ be arbitrary constants $D_i\geq0$ and let $\pi$ be an arbitrary distribution with finite second moment. Suppose $(\mu_0,\mu_1,...)$ is a sequence of measures satisfying for $k\geq1$ and $n=k\;\mathrm{mod}\;{b}$
    \begin{equation}
        \KL{\mu_{t_k}}{\pi_\beta}\leq C_n\KL{\mu_{t_{k-1}}}{\pi}+(1-C_n)\KL{\mu_{t_{k-1},\overline{B}_n}}{\pi_{\overline{B}_n}} + D_n.
    \end{equation}
    Then we can bound
    \begin{equation}
        \KL{\mu_{t_{kb}}}{\pi}\leq C_{\max}^{k}\KL{\mu_{t_{(k-1)b}}}{\pi}+\sum_{i=1}^bD_i
    \end{equation}
    where $C_{\max}=\max \{C_1,...,C_b\}$.
\end{lemma}

~\cite{lee_gibbs_2022} \emph{lower} bounded the KL-divergence in Bayesian coordinate ascent variational inference by similarly comparing the change in $\kl$ across conditioned steps. However, their focus was on inference over mean-field parametric distributions rather than the broader class of LSI Gibbs measures, making Lemma~\ref{lemma:cyclic_kl_general} a stronger result.

The convergence of CBLD follows by choosing $D_i=0$, $C_{\max}=e^{-2\gamma\beta^{-1}\lambda_{\min}}$:
\begin{theorem}[CBLD $\KL{\mu(t_k)}{\pi_\beta}$ Convergence]\label{theorem:CBLD_convergence}
    Let $\sigma=(B_1,...,B_b)$ be a given block permutation and let $\lambda=(\lambda_1,...,\lambda_b)\in\R^d$ be the sampling times for each block, $\lambda_i>0$. For any $\pi_\beta\propto\exp[-\beta f(x)]$ satisfying Assumptions~\ref{assmp:diff} and~\ref{assmp:lsi}, and any $\beta>0$, the sample distribution after $kb$ steps of CBLD ($\mu_{t_{kb}}$) satisfies
    \begin{equation}\label{eqn:cbld_bound}
        \KL{\mu_{t_{kb}}}{\pi_\beta}\leq e^{-2\gamma\beta^{-1}\lambda_{\min}k}\KL{\mu_0}{\pi_\beta}.
    \end{equation}
    
\end{theorem}

When $D_i\neq 0$, Lemma~\ref{lemma:cyclic_kl_general} accounts for biased sampling algorithms, such as Langevin Monte Carlo (LMC). Accordingly, we combine Lemma~\ref{lemma:cyclic_kl_general} with existing LSI bounds for LMC from \cite{vempala_rapid_2019,chewi_analysis_2021} to prove convergence for a discrete time ``cyclic block Langevin Monte Carlo'' in Appendix~\ref{sec:cyclic_appdx}.

\paragraph*{Observations:}
For RBLD and CBLD, the convergence is limited by the shortest step duration $\lambda_{\min}$ and minimal block probability $\phi_{\min}$. For constant block sizes, the optimal choice for both CBLD and RBLD is therefore constant $\lambda_i=\lambda_j=\lambda$ and uniform $\phi_i=1/b$. This contrasts discrete-time block optimization, where distinct step sizes/probability distributions provide advantage on ill-conditioned problems~\citep{nesterov_efficiency_2012,beck_convergence_2013,ding_random_2021} due to the effect of varying Lipschitz constants in discretization error terms. In the case of constant $\lambda$ with uniform $\phi$, RBLD and CBLD have identical descent bounds, as we numerically demonstrate in Section~\ref{sec:numerical}. This considerably simplifies hyperparameter selection \emph{for ideal devices}, reducing from $O(b)$ parameters to 1 ($\lambda$). In the following section we continue to assume a constant step duration $\lambda$ for simplicity, though future analyses may reveal more optimized step size selections for finite-variation devices.

\subsection{Finite Variation}\label{sec:variation}
Theorems~\ref{thm:rbld_kl} and ~\ref{theorem:CBLD_convergence} provide optimistic lower bounds for DX sampling, however a real machine will have analog errors perturbing the target function~\citep{albash_analog_2019,melanson_thermodynamic_2023}. As a generalization of ~\cite{albash_analog_2019}, we model a DX with analog variation with a ``perturbed'' gradient oracle $g_\delta(x):\R^d\to\R^d$, where $\delta\in\bm{D}$ denotes a fixed perturbation from arbitrary domain $\bm{D}$.  Unlike stochastic optimization, which assumes that the perturbation changes with each gradient evaluation, DX perturbations are fixed for each device. To provide guarantees under device variation, we need to restrict the perturbations and functions permitted:

\begin{assumption}\label{assmp:grad_var}
    For fixed $\delta\in \bm{D}$, there exist constants $M, B\geq 0$ such that for all $ x\in \R^d$
    \[\Vert \nabla f(x)-g_\delta(x)\Vert ^2\leq M^2\Vert x\Vert ^2+B^2.\]
\end{assumption}
\begin{assumption}\label{assmp:lipschitz}
    $f$ is $L$-smooth and, for fixed $\delta\in \bm{D}$, $g_\delta$ is $G$-Lipschitz continuous. That is, for all $x,y\in\R^d$
    \[\Vert \nabla f(x)-\nabla f(y)\Vert  \leq L\Vert x-y\Vert, \]
    \[\Vert g_\delta(x)-g_\delta(y)\Vert  \leq G\Vert x-y\Vert. \]
\end{assumption}

\begin{assumption}\label{assmp:dissapative}
    $f$ and $g_\delta$ are $(m,c)$-dissipative and $(\mathfrak{m},\mathfrak{c})$-dissipative respectively, i.e., there exists positive constants $m> 0$, $c$, $\mathfrak{m}$, $\mathfrak{c}\geq 0$ such that for all $x\in\R^d$:
    \[\inner{\nabla f(x)}{x}\geq m\Vert x\Vert ^2-c,\]
    \[\inner{\nabla g_\delta(x)}{x}\geq \mathfrak{m}\Vert x\Vert ^2-\mathfrak{c}.\]
\end{assumption}

Assumption~\ref{assmp:grad_var} limits the Euclidean distance between $\nabla f$ and $g_\delta$, with the constants $M$ and $B$ appearing in later bounds. Assumption~\ref{assmp:dissapative} is a common assumption in analyses of stochastic gradient sampling algorithms~\citep{raginsky_non-convex_2017, li_sharp_2022,zhang_nonasymptotic_2023}. Specifically, it enables us to bound the ideal Langevin second moment $\E \Vert y(t_k)\Vert ^2$ in the proof of Theorem~\ref{theorem:w2_conv}. 
Assumption~\ref{assmp:lipschitz} is not directly used in our proofs, but is required for a Girsanov change of measure. Assumptions~\ref{assmp:grad_var} and~\ref{assmp:dissapative} both restrict the type of perturbation with Assumption~\ref{assmp:dissapative} also limiting the magnitude. If Assumptions~\ref{assmp:lipschitz} and~\ref{assmp:dissapative} hold for the target potential, it is reasonable to expect that they hold for the perturbed oracle as well, since DX variation typically manifests as additive or multiplicative perturbations in analog components implementing $\nabla f$~\cite{xiao_accuracy_2022,aifer_thermodynamic_2023}.

Take the example of a Gaussian potential $f(x)=\frac{1}{2}x^\top\Sigma^{-1}x$ with $g_\delta(x)=\Sigma^{-1}\circ(1+\delta)x$, where $\delta\in\R^{d\times d}$ is a ``perturbation matrix'' with $\delta_{ij}\sim \mathcal{N}(0,\Delta^2)$ and $\circ$ denotes a component-wise Hadamard product. Regardless of the standard deviation $\Delta$, we satisfy Assumptions~\ref{assmp:grad_var} and~\ref{assmp:lipschitz} with $M$ and $L$ both equal to the maximal magnitude eigenvalue of $\delta$ and $\Sigma^{-1}$ respectively with $B=0$. However, if $\Sigma^{-1}(1+\delta)$ has negative eigenvalues there is no $\mathfrak{m}>0$ satisfying Assumption~\ref{assmp:dissapative}, placing an upper limit on the perturbation strength.

\begin{assumption}\label{assmp:init_density}
    The density of the initial law $\mu_0$ satisfies
    \[\kappa_0\triangleq \log\int_{\R^d}e^{\Vert w\Vert ^2}d\mu_0<\infty.\]
\end{assumption}
In practice, dynamical accelerators typically operate over bounded domains, such as the unit hypercube~\citep{afoakwa_brim_2021} or unit circle~\citep{wang_oim_2019,inagaki_coherent_2016}, hence the iterate magnitude is bounded in any case. However, bounding over the entire space would provide insufficiently tight upper bounds and our methodology assumes that the measures are supported on $\R^d$. We leave consideration of domains with bounded support to future work, potentially applying methods from reflected Langevin diffusion theory~\citep{bubeck_sampling_2018} or projected differential analysis~\citep{cherukuri_asymptotic_2016}. 

We begin by stating the following bound on the distance between the measures of ideal and perturbed BLD, proved in Appendix~\ref{appdx:stochastic}:

\begin{lemma}[Finite Variation Block Langevin $W_2$ Distance]\label{lemma:w2_dist}
Let $x(t)$, $y(t)$ be non-ideal and ideal block Langevin processes respectively with associated distributions $\mu_t, \nu_t$.  For any $\pi_\beta\propto\exp[-\beta f(x)]$ satisfying~\ref{assmp:diff},~\ref{assmp:lsi}, and ~\ref{assmp:dissapative} with $\beta>\frac{2}{m}$, $g_\delta(x)$ satisfying Assumption~\ref{assmp:grad_var}, $\mu_0$ satisfying Assumption~\ref{assmp:init_density}, and $k\lambda>1$ after $kb$ iterations of BLD
    \begin{equation}\label{eqn:w2_dist}
        W_2(\mu_{t_{kb}},\nu_{t_{kb}})\leq \sqrt{C_0\biggl[(C_1 + \sqrt{C_1})+(C_2 + \sqrt{C_2})\sqrt{\lambda}\biggr]}k\lambda
    \end{equation}
    where $C_0$, $C_1$, and $C_2$ are given in Appendix~\ref{appdx:stochastic}.
\end{lemma}

From previous discussions, setting $\phi_i=1/b$, $\lambda_{i}=\lambda$ unifies the bounds for RBLD and CBLD. In this regime, we can prove the following statement as a simple consequence of the triangle inequality $W_2(\mu,\nu)\leq W_2(\mu,\eta)+W_2(\eta,\nu)$ and the Otto-Villani theorem
\begin{theorem}[Finite-Variation BLD $W_2$ Convergence]\label{theorem:w2_conv}
    \begin{equation}
    \begin{split}
        W_2(\mu_{t_{bk}},\pi_\beta)\leq& \left(\frac{2}{\gamma}\right)^{1/2} e^{-\gamma\beta^{-1}\lambda k}\sqrt{\KL{\mu_0}{\pi_\beta}}+\sqrt{C_0\biggl[(C_1 + \sqrt{C_1})+(C_2 + \sqrt{C_2})\sqrt{\lambda}\biggr]}k\lambda.
    \end{split}
    \end{equation}

\end{theorem}

Following ~\cite{raginsky_non-convex_2017}, if we choose $k\lambda=\frac{\beta}{\gamma}\log{\frac{2\sqrt{2\KL{\mu_0}{\pi_\beta}}}{\varepsilon\sqrt{\gamma}}}$ and set $\lambda\leq (\varepsilon \gamma)^{4}\left(\beta\log\left[{{2\sqrt{2\KL{\mu_0}{\pi_\beta}}}/{(\sqrt{\gamma}\varepsilon)}}\right])\right)^{-4}$, we have
    \begin{equation}
    \begin{split}\label{eqn:bias}
        W_2(\mu_{t_{bk}},\pi_\beta)\leq&\frac{\varepsilon}{2}+\sqrt{C_0}\biggl[\sqrt{C_1 + \sqrt{C_1}}\frac{\beta}{\gamma}\log{\frac{2\sqrt{2\KL{\mu_0}{\pi_\beta}}}{\varepsilon\sqrt{\gamma}}}+\varepsilon\sqrt{C_2 + \sqrt{C_2}}\biggr].
    \end{split}
    \end{equation}
    We thereby obtain a total bound on the Wasserstein error $\mathcal{O}(\log{\frac{1}{\varepsilon}}+\varepsilon)$ for arbitrary $\varepsilon > 0$. 
    \paragraph*{Observations:}Similar to discrete-time LMC, our Wasserstein bound has a non-zero lower bound with respect to $\varepsilon$: non-ideal devices introduce \textit{bias}. Unlike discrete LMC, the bias in Equation~\eqref{eqn:bias} does not result from a forward-flow discretization~\citep{wibisono_sampling_2018,chewi_analysis_2021}. Instead, the constants $C_0, C_1,C_2$ are solely due to finite analog variation. For $M=0$, $B=0$, we recover exponential, unbiased convergence in $W_2$. However, akin to LMC, practitioners can select the step size $\lambda$ and the injected noise $\beta$ to control the bias. Higher temperatures (lower $\beta$) result in a lower bias, as expected from the application of a noisy channel in measure space. Moreover, DX users/designers typically characterize $M$, $B$ during device calibration: simultaneously lowering the impact of analog non-ideality and allowing for a rough bound on the distribution bias~\citep[See Section C.2.a of ][]{melanson_thermodynamic_2023}.
    
    A Wasserstein bound suffices as a performance guarantee in sampling tasks such as Boltzmann machine inference~\citep{hinton2006fast} or statistical physics simulation~\citep{hamerly_experimental_2019,ng_efficient_2022,inaba_thermodynamic_2023}. For optimization, assuming quadratic function growth with $\beta\geq \frac{2}{m}$ and a dissapative gradient oracle (Assumption~\ref{assmp:dissapative}) allows the use of a continuity inequality ~\citep[Lemma 6 of ][]{raginsky_non-convex_2017} and second moment bound~\citep[Proposition 11 of][]{raginsky_non-convex_2017} to bound $\E_{\mu_{t_k}}[f(x)]-\E_{\pi_\beta}[f(x)]$ and $\E_{\pi_\beta}[f(x)]-\min_{x\in\R^d}f(x)$ respectively
    \begin{equation}\label{eqn:emp_exp_dist}
        \E_{\mu_{t_{bk}}}[f(x)]-\E_{\pi_\beta}[f(x)]\leq (M\sigma+B)W_2(\mu_{t_{bk}},\pi_\beta),
    \end{equation}
    \begin{equation}\label{eqn:gibbs_exp_dist}
        \E_{\pi_\beta}[f(x)]-\min_{x\in\R^d}f(x)\leq \frac{d}{2\beta}\log\left(\frac{eL}{m}\left(\frac{c\beta}{d}+1\right)\right)
    \end{equation}
    where $\sigma^2=\max\{\E_{\mu_{t_{bk}}}[x^2], \E_{\pi_\beta}[x^2]\}$ (given in Appendix~\ref{appdx:stochastic}). Combining Equations~\eqref{eqn:emp_exp_dist} and~\eqref{eqn:gibbs_exp_dist}, we obtain
    \begin{equation}\label{eqn:exp_conv}
    \begin{split}
        \E_{\mu_{t_{bk}}}[f(x)]-\min_{x\in\R^d}f(x)&\leq \frac{d}{2\beta}\log\left(\frac{eL}{m}\left(\frac{c\beta}{d}+1\right)\right) + (M\sigma+B)W_2(\mu_{t_{bk}},\pi_\beta)\\   
        &=\mathcal{O}(\frac{d}{\beta}\log {\beta}{d}+(M+B)(\varepsilon +\log\varepsilon^{-1})).
    \end{split}
    \end{equation}
    Controlling the first term requires increasing $\beta$ (lower injected noise) in tandem with problem dimension. Conversely, controlling the second term requires $\lambda kb\propto\beta$, $\lambda\propto \frac{1}{\beta^4}$, i.e., more iterations with lower step duration with increasing $\beta$. In digital algorithms, we are free to choose $\lambda$ arbitrarily small to meet given precision requirements (though the program convergence may be impracticably slow). Dynamical accelerators typically have a lower bound on $\lambda$ (e.g. a digital clock), translating into an effective upper bound on $\beta$. 
\section{Numerical Experiments}\label{sec:numerical}

As an illustrative example, we simulated CBLD and RBLD behavior for $d=50$ Gaussian sampling with $\beta=1$ and uniform $\lambda_i=\lambda$\footnote{Code and data can be found at~\url{https://github.com/ur-acal/BlockLangevin}.}. Gaussian distributions permit closed-form solutions for $\KL{\mathcal{N}(u_1, \Sigma_1)}{\mathcal{N}(u_2, \Sigma_2)}$, allowing for a quantitative estimate of convergence.
Moreover, several proposed use cases for DXs rely on Gaussian sampling, including matrix inversion~\citep{aifer_thermodynamic_2023} and uncertainty quantification~\citep{melanson_thermodynamic_2023}. Other works have also proposed using DXs to optimize strongly-convex functions of the form $f(x)=(x-u)^\top W(x-u)$~\citep{wuextending,song_ds-gl_2024}. As discussed in the preceding section, our bounds provide expected function gap guarantees from sampling $\pi=\mathcal{N}(u,2\beta^{-1} W^{-1})$, where optimization would occur in the $\beta\to\infty$ limit. Appendix~\ref{sec:experimental_methods} gives more experimental details.

\begin{wrapfigure}{r}{0.58\textwidth}
    \centering
    \vskip -5pt
\begin{subfigure}[]{0.49\linewidth}
    \includegraphics[width=\linewidth]{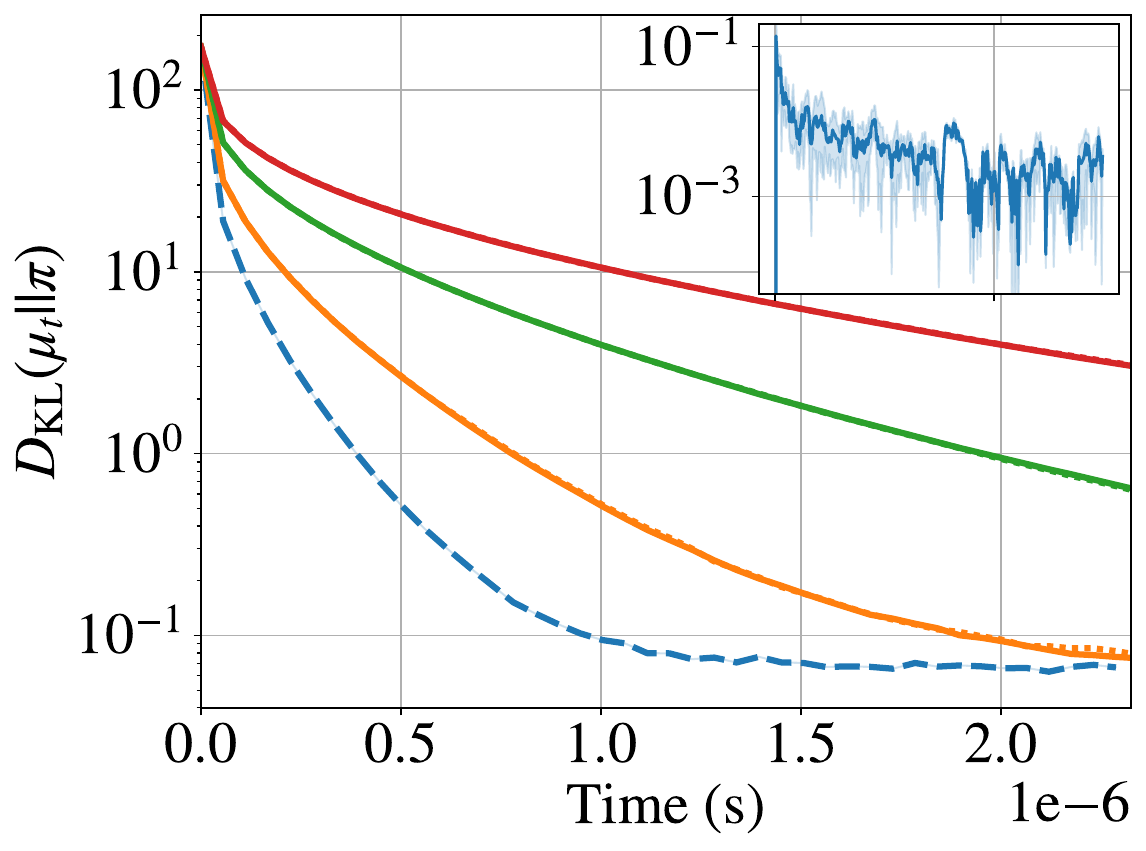}
    \vskip -5pt
    \subcaption{}
    \label{fig:gaussian_time}
\end{subfigure}
\hfill
\begin{subfigure}[]{0.49\linewidth}
    \includegraphics[width=\linewidth]{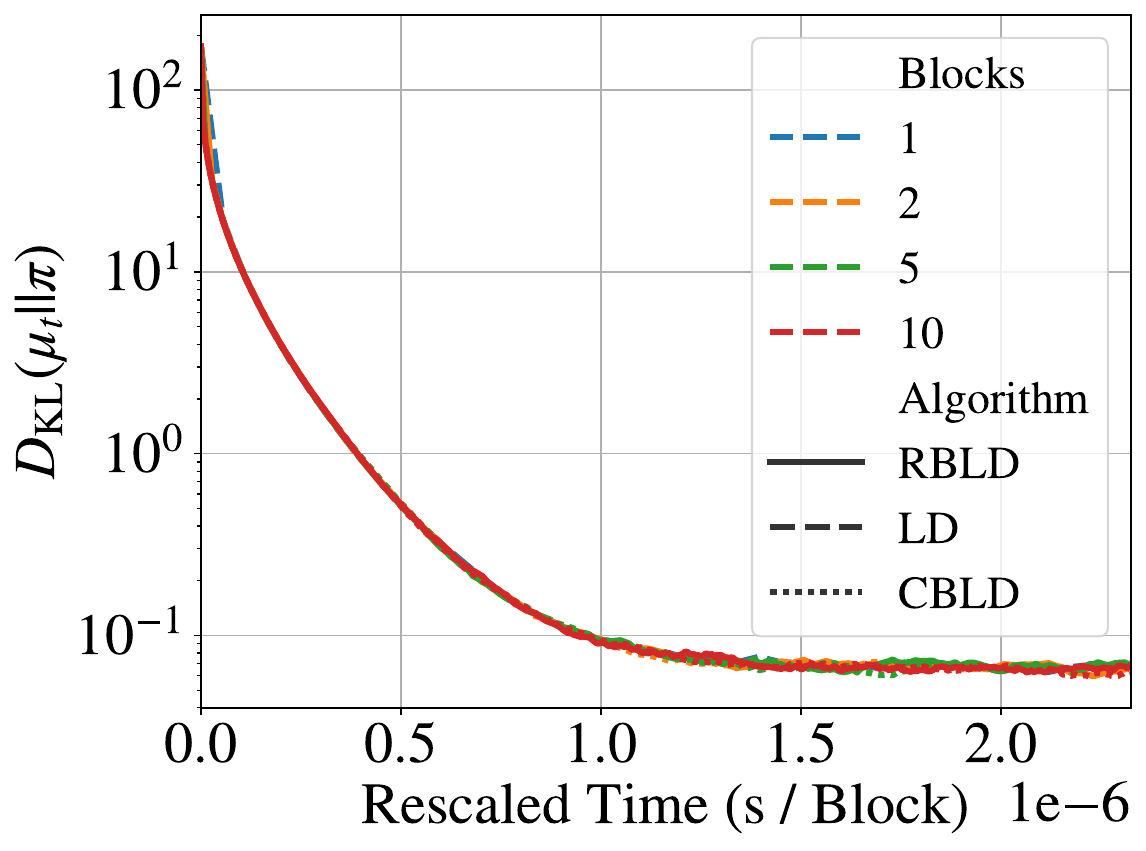}
    \vskip -5pt
    \subcaption{}
    \label{fig:gaussian_cycles}
\end{subfigure}
    \vskip -5pt
\caption{ Convergence in $\kl$ for varying block counts $b$ with for $b-$BCLD and $b-$RCLD \textbf{(a)} versus simulated time and \textbf{(b)} versus cycles $kb$. The inset plot in \textbf{(a)} shows the absolute difference between the RBLD and CBLD $\kl$ values averaged over each block count. }
    \vskip -5pt
\end{wrapfigure}

Fig.~\ref{fig:gaussian_time} shows the convergence in $\kl$ for the {estimated} distribution $\mu_{t}$ with block counts $b\in\{1,2,5,10\}$ versus simulated time. Note that we are only comparing the convergence \emph{relative to ideal LD} ($b=1$) rather than the exact rate or $\kl$ value (see Appendix~\ref{sec:experimental_methods}). The block methods converge slower than the full-gradient process, with the rate of convergence decreasing with more blocks. By comparing the exponents in Theorems~\ref{thm:ld_conv} ($-2\gamma\beta^{-1} t$) and~\ref{thm:rbld_kl} ($-2\gamma \beta^{-1}\lambda\phi_{\min}k$), we note that choosing $k=\frac{t}{\phi_{\min}\lambda}$ makes the two contractions equal. Similarly, the exponent in Theorem~\ref{theorem:CBLD_convergence} for a contraction in $b$ iterations $(-2\gamma\beta^{-1} \lambda k)$ suggests the choice $kb=\frac{t}{\lambda}$ iterations. Since $\lambda $ is the amount of time, spent per block, this suggests that equating the total time per block should result in roughly equivalent contractions.  Fig.~\ref{fig:gaussian_cycles} confirms this prediction by showing the same data after rescaling the x-axis by $b$ for each method, with the curves converging within sampling error. Block sampling therefore incurs an $O(b)$ slowdown in real time, similar to coordinate methods in optimization (without accounting for smoothness~\citep{wright_coordinate_2015}).
\pagebreak

\begin{wrapfigure}{o}{0.58\textwidth}
        \centering
\begin{subfigure}[t]{0.49\linewidth}
    \includegraphics[width=\linewidth]{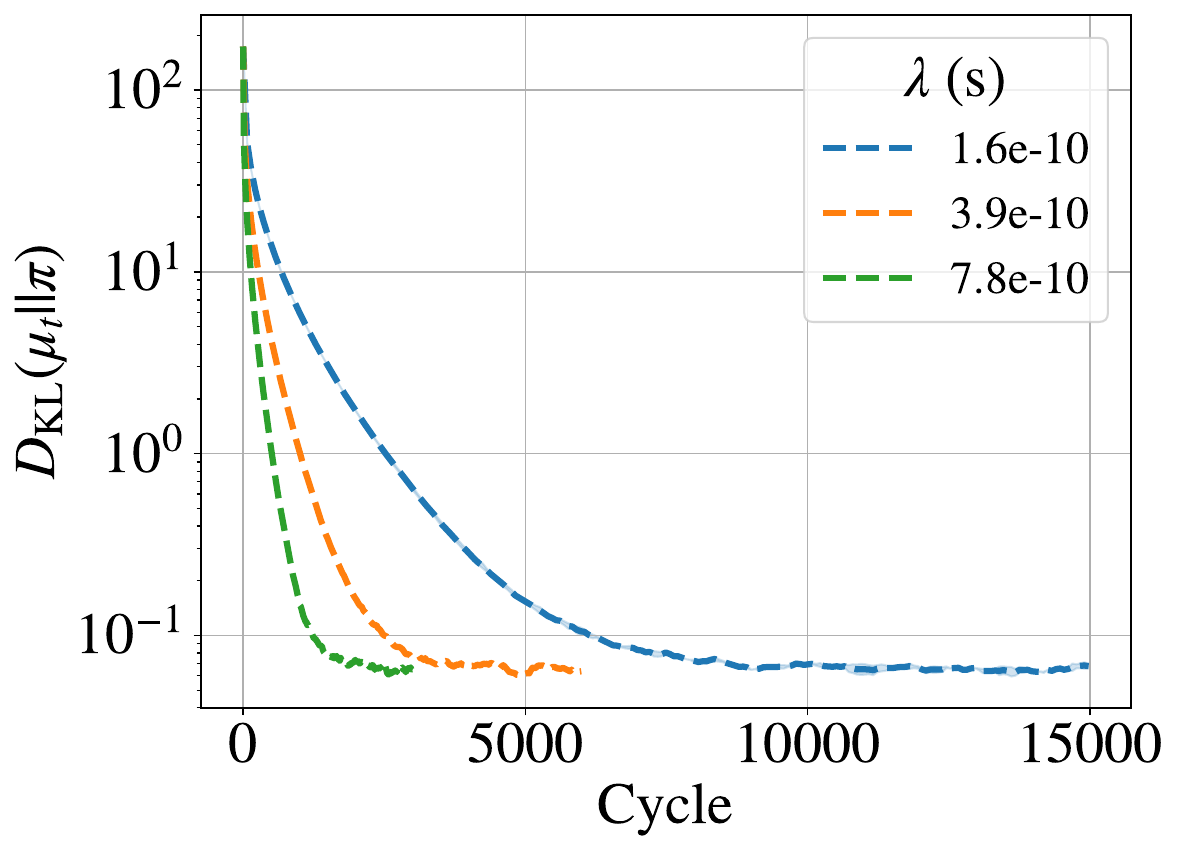}
    \subcaption{}
    \label{fig:gaussian_step}
\end{subfigure}
\begin{subfigure}[t]{0.49\linewidth}
    \includegraphics[width=\linewidth]{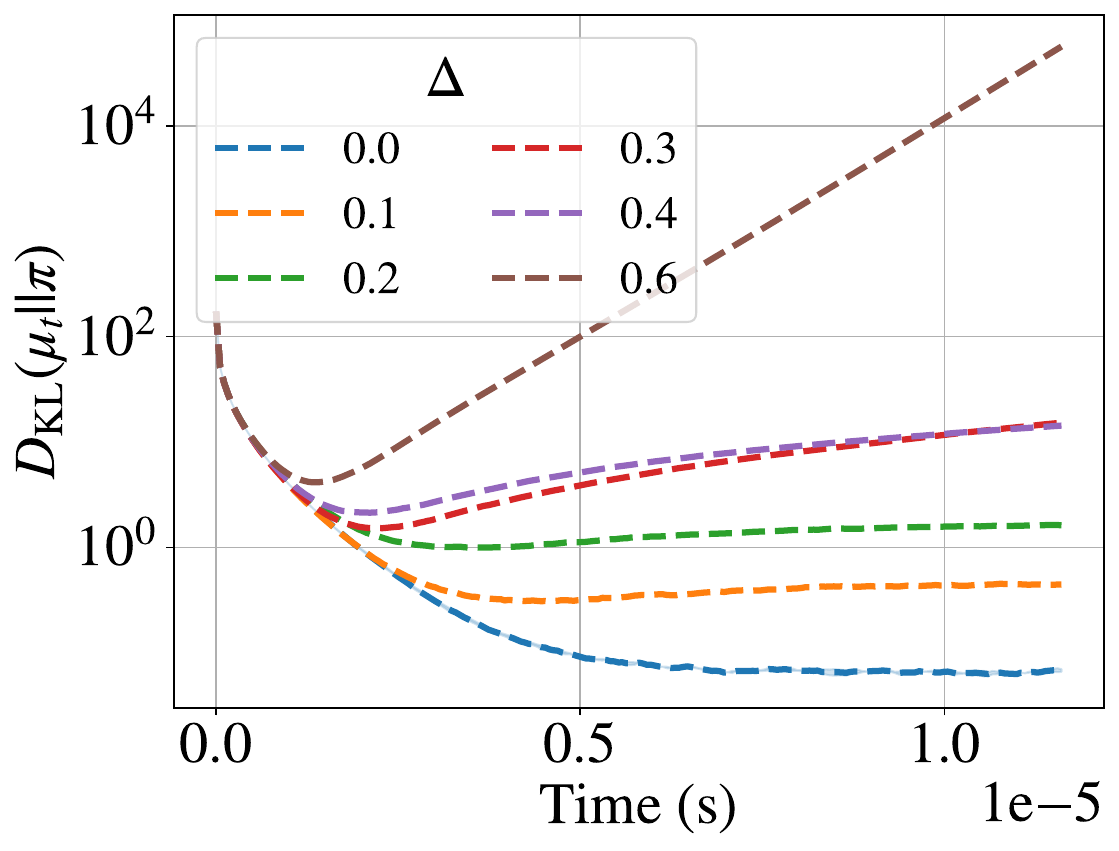}
    \subcaption{}
    \label{fig:gauss_pert}
\end{subfigure}
\caption{BCLD $\kl$ convergence \textbf{(a)} versus whole-problem cycles with varying block duration $\lambda$ and \textbf{(b)} versus simulated time with varying  multiplicative Gaussian perturbations. }
\label{fig:numerical}
\end{wrapfigure}

Our testing showed that all step sizes lead to the same convergence rate with respect to time (not shown). However, larger step sizes lead to larger decay w.r.t. whole problem \textit{cycles} ($b$ iterations), as shown in Fig.~\ref{fig:gaussian_step}. This reinforces the importance of \emph{per block} sampling time for continuous coordinate methods. Finally, we perturb the similarity matrix $\Sigma^{-1}$ with component-wise variation $\tilde\Sigma_{ij}=\Sigma_{ij}(1+\delta_{ij})$, $\delta_{ij}\sim\mathcal{N}(0,\Delta)$. Fig.~\ref{fig:gauss_pert} shows the impact on $\kl$ convergence with increasing perturbation strength. At $\Delta=0.6$, $\Sigma$ is no longer positive-definite, causing the iterate to diverge (not shown on plot). $\Delta\in\{0.1, 0.2, 0.3, 0.4, 0.5\}$ satisfied Assumption~\ref{assmp:dissapative} and were therefore stable (though biased), but $\Delta=0.6$ did not, in line with discussion in the preceding section.

\section{Conclusion}

\paragraph{Findings:}
In this work, we prove novel bounds for Large Neighborhood Local Search (LNLS) sampling algorithms leveraging continuous Langevin diffusion, providing valuable information to device designers, potential DX adopters, and future analyses. Specifically, we
\begin{enumerate}
    \item prove novel non-asymptotic convergence bounds for randomized and cyclic block selection strategies, finding that both methods produce identical convergence rates in KL-divergence.
    \item provide probabilistic convergence bounds for LNLS sampling using non-ideal analog devices, with biases expressed using practically measurable/estimable constants
    \item validate our theoretical results by numerical simulation, demonstrating \ding{172} the expected equivalence of cyclic and randomized strategies and \ding{173} the expected dependence of convergence rates on algorithm/problem parameters.
\end{enumerate}


\paragraph{Limitations:}
Assumptions~\ref{assmp:lsi} and~\ref{assmp:dissapative} provide useful bounds for many ML and optimization problems over continuous domains. However, DX applications include discrete choice problems and/or significantly non-convex potentials, such as mixed integer programming. Future bounds necessarily involve more general assumptions than the $\gamma$-LSI class considered here. Analog accelerators also typically use low-precision ($<8b$) DACs and ADCs for input/output~\citep{xiao_accuracy_2022}, making studies of quantizated convergence/expected function gap critical for real-world applications.

A technical limitation is the assumption of support over $\R^d$. DXs operate in bounded subspaces such as the unit circle~\citep{inagaki_coherent_2016} or hypercube~\citep{afoakwa_brim_2021}. Future work applying projected differential analysis~\citep{cherukuri_asymptotic_2016,bubeck_sampling_2018} is therefore needed.

\paragraph{Directions for Future Work:}
In this work, we focused on the "inference" stage where the coupling parameters are fixed. For training, current usages of DX systems update model weights in discrete time, either by the digital controller~\citep{song_ds-gl_2024} or by a discrete step in the analog domain~\citep{vengalam_supporting_2023}. It will be a boon to DX research to further develop theories for the training regime where model weights become the end goal for dynamics.

Finally, this work focuses on a simplified LNLS framework, relevant to the 
dynamics-accelerated LNLS popular in literature~\citep{sharma_increasing_2022, raymond_hybrid_2023, wuextending}. It leaves open the question whether additional digital steps, such as a Metropolis-Hastings filter or replica exchange, could improve the non-asymptotic accuracy or convergence rate.
\pagebreak
\section*{Acknowledgments}
This work was supported in part by NSF under Awards No. 2231036 and No. 2233378; and by
DARPA under contract No. FA8650-23-C-7312. Our thanks go to Prof. Jiaming Liang for several
fruitful discussions and helpful feedback on an early draft of this work.

\bibliographystyle{iclr2025_conference}
\bibliography{refs}

\begin{thebibliography}{76}
\providecommand{\natexlab}[1]{#1}
\providecommand{\url}[1]{\texttt{#1}}
\expandafter\ifx\csname urlstyle\endcsname\relax
  \providecommand{\doi}[1]{doi: #1}\else
  \providecommand{\doi}{doi: \begingroup \urlstyle{rm}\Url}\fi

\bibitem[Afoakwa et~al.(2021)Afoakwa, Zhang, Vengalam, Ignjatovic, and
  Huang]{afoakwa_brim_2021}
Richard Afoakwa, Yiqiao Zhang, Uday Kumar~Reddy Vengalam, Zeljko Ignjatovic,
  and Michael Huang.
\newblock {BRIM}: {Bistable} {Resistively}-{Coupled} {Ising} {Machine}.
\newblock In \emph{2021 {IEEE} {International} {Symposium} on
  {High}-{Performance} {Computer} {Architecture} ({HPCA})}, pp.\  749--760,
  February 2021.
\newblock \doi{10.1109/HPCA51647.2021.00068}.
\newblock ISSN: 2378-203X.

\bibitem[Ahuja et~al.(2002)Ahuja, Ergun, Orlin, and Punnen]{ahuja_survey_2002}
Ravindra~K. Ahuja, Özlem Ergun, James~B. Orlin, and Abraham~P. Punnen.
\newblock A survey of very large-scale neighborhood search techniques.
\newblock \emph{Discrete Applied Mathematics}, 123\penalty0 (1):\penalty0
  75--102, November 2002.
\newblock ISSN 0166-218X.
\newblock \doi{10.1016/S0166-218X(01)00338-9}.
\newblock URL
  \url{https://www.sciencedirect.com/science/article/pii/S0166218X01003389}.

\bibitem[Aifer et~al.(2023)Aifer, Donatella, Gordon, Ahle, Simpson, Crooks, and
  Coles]{aifer_thermodynamic_2023}
Maxwell Aifer, Kaelan Donatella, Max~Hunter Gordon, Thomas Ahle, Daniel
  Simpson, Gavin~E. Crooks, and Patrick~J. Coles.
\newblock Thermodynamic {Linear} {Algebra}, August 2023.
\newblock URL \url{http://arxiv.org/abs/2308.05660}.
\newblock arXiv:2308.05660 [cond-mat, physics:quant-ph].

\bibitem[Albash et~al.(2019)Albash, Martin-Mayor, and Hen]{albash_analog_2019}
Tameem Albash, Victor Martin-Mayor, and Itay Hen.
\newblock Analog errors in {Ising} machines.
\newblock \emph{Quantum Science and Technology}, 4\penalty0 (2):\penalty0
  02LT03, April 2019.
\newblock ISSN 2058-9565.
\newblock \doi{10.1088/2058-9565/ab13ea}.
\newblock URL \url{https://dx.doi.org/10.1088/2058-9565/ab13ea}.
\newblock Publisher: IOP Publishing.

\bibitem[Albertsson \& Rusu(2023)Albertsson and Rusu]{albertsson_ising_2023}
Dagur~I. Albertsson and Ana Rusu.
\newblock Ising {Machine} {Based} on {Bifurcations} in a {Network} of {Duffing}
  {Oscillators}.
\newblock In \emph{2023 {IEEE} {International} {Symposium} on {Circuits} and
  {Systems} ({ISCAS})}, pp.\  1--5, Monterey, CA, USA, May 2023. IEEE.
\newblock ISBN 978-1-66545-109-3.
\newblock \doi{10.1109/ISCAS46773.2023.10181810}.
\newblock URL \url{https://ieeexplore.ieee.org/document/10181810/}.

\bibitem[Arnese \& Lacker(2024)Arnese and Lacker]{arnese_convergence_2024}
Manuel Arnese and Daniel Lacker.
\newblock Convergence of coordinate ascent variational inference for
  log-concave measures via optimal transport, April 2024.
\newblock URL \url{http://arxiv.org/abs/2404.08792}.
\newblock arXiv:2404.08792 [cs, math, stat].

\bibitem[Bashar \& Shukla(2023)Bashar and Shukla]{bashar_designing_2023}
Mohammad~Khairul Bashar and Nikhil Shukla.
\newblock Designing {Ising} machines with higher order spin interactions and
  their application in solving combinatorial optimization.
\newblock \emph{Scientific Reports}, 13\penalty0 (1):\penalty0 9558, June 2023.
\newblock ISSN 2045-2322.
\newblock \doi{10.1038/s41598-023-36531-4}.
\newblock URL \url{https://www.nature.com/articles/s41598-023-36531-4}.
\newblock Publisher: Nature Publishing Group.

\bibitem[Bauerschmidt \& Bodineau(2019)Bauerschmidt and
  Bodineau]{bauerschmidt_very_2019}
Roland Bauerschmidt and Thierry Bodineau.
\newblock A very simple proof of the {LSI} for high temperature spin systems.
\newblock \emph{Journal of Functional Analysis}, 276\penalty0 (8):\penalty0
  2582--2588, April 2019.
\newblock ISSN 0022-1236.
\newblock \doi{10.1016/j.jfa.2019.01.007}.
\newblock URL
  \url{https://www.sciencedirect.com/science/article/pii/S0022123619300278}.

\bibitem[Beck \& Tetruashvili(2013)Beck and
  Tetruashvili]{beck_convergence_2013}
Amir Beck and Luba Tetruashvili.
\newblock On the {Convergence} of {Block} {Coordinate} {Descent} {Type}
  {Methods}.
\newblock \emph{SIAM Journal on Optimization}, 23\penalty0 (4):\penalty0
  2037--2060, January 2013.
\newblock ISSN 1052-6234.
\newblock \doi{10.1137/120887679}.
\newblock URL \url{https://epubs.siam.org/doi/10.1137/120887679}.
\newblock Publisher: Society for Industrial and Applied Mathematics.

\bibitem[Bhattacharya et~al.(2023)Bhattacharya, Pati, and
  Yang]{bhattacharya_convergence_2023}
Anirban Bhattacharya, Debdeep Pati, and Yun Yang.
\newblock On the {Convergence} of {Coordinate} {Ascent} {Variational}
  {Inference}, June 2023.
\newblock URL \url{http://arxiv.org/abs/2306.01122}.
\newblock arXiv:2306.01122 [cs, math, stat].

\bibitem[Bolley \& Villani(2005)Bolley and Villani]{bolley2005weighted}
Fran{\c{c}}ois Bolley and C{\'e}dric Villani.
\newblock Weighted csisz{\'a}r-kullback-pinsker inequalities and applications
  to transportation inequalities.
\newblock In \emph{Annales de la Facult{\'e} des sciences de Toulouse:
  Math{\'e}matiques}, volume 14(3), pp.\  331--352, 2005.

\bibitem[Bolte et~al.(2010)Bolte, Daniilidis, Ley, and
  Mazet]{bolte2010characterizations}
J{\'e}r{\^o}me Bolte, Aris Daniilidis, Olivier Ley, and Laurent Mazet.
\newblock Characterizations of {\l}ojasiewicz inequalities: subgradient flows,
  talweg, convexity.
\newblock \emph{Transactions of the American Mathematical Society},
  362\penalty0 (6):\penalty0 3319--3363, 2010.

\bibitem[Booth et~al.(2017)Booth, Reinhardt, and
  Roy]{booth_partitioning_nodate}
Michael Booth, Steven~P Reinhardt, and Aidan Roy.
\newblock Partitioning {Optimization} {Problems} for {Hybrid}
  {Classical}/{Quantum} {Execution}.
\newblock Technical report, D-Wave, 2017.

\bibitem[Brown et~al.(2024)Brown, Venturelli, Pavone, and
  Neira]{brown_accelerating_2024}
Robin Brown, Davide Venturelli, Marco Pavone, and David E.~Bernal Neira.
\newblock Accelerating {Continuous} {Variable} {Coherent} {Ising} {Machines}
  via {Momentum}, January 2024.
\newblock URL \url{http://arxiv.org/abs/2401.12135}.
\newblock arXiv:2401.12135 [quant-ph].

\bibitem[Bubeck et~al.(2018)Bubeck, Eldan, and Lehec]{bubeck_sampling_2018}
Sébastien Bubeck, Ronen Eldan, and Joseph Lehec.
\newblock Sampling from a {Log}-{Concave} {Distribution} with {Projected}
  {Langevin} {Monte} {Carlo}.
\newblock \emph{Discrete \& Computational Geometry}, 59\penalty0 (4):\penalty0
  757--783, June 2018.
\newblock ISSN 1432-0444.
\newblock \doi{10.1007/s00454-018-9992-1}.
\newblock URL \url{https://doi.org/10.1007/s00454-018-9992-1}.

\bibitem[Bybee et~al.(2023)Bybee, Kleyko, Nikonov, Khosrowshahi, Olshausen, and
  Sommer]{bybee_efficient_2023}
Connor Bybee, Denis Kleyko, Dmitri~E. Nikonov, Amir Khosrowshahi, Bruno~A.
  Olshausen, and Friedrich~T. Sommer.
\newblock Efficient optimization with higher-order {Ising} machines.
\newblock \emph{Nature Communications}, 14\penalty0 (1):\penalty0 6033,
  September 2023.
\newblock ISSN 2041-1723.
\newblock \doi{10.1038/s41467-023-41214-9}.
\newblock URL \url{https://www.nature.com/articles/s41467-023-41214-9}.
\newblock Publisher: Nature Publishing Group.

\bibitem[Chak et~al.(2023)Chak, Kantas, and Pavliotis]{chak_generalized_2023}
Martin Chak, Nikolas Kantas, and Grigorios~A. Pavliotis.
\newblock On the {Generalized} {Langevin} {Equation} for {Simulated}
  {Annealing}.
\newblock \emph{SIAM/ASA Journal on Uncertainty Quantification}, 11\penalty0
  (1):\penalty0 139--167, March 2023.
\newblock \doi{10.1137/21M1462970}.
\newblock URL \url{https://epubs.siam.org/doi/full/10.1137/21M1462970}.
\newblock Publisher: Society for Industrial and Applied Mathematics.

\bibitem[Cherukuri et~al.(2016)Cherukuri, Mallada, and
  Cortés]{cherukuri_asymptotic_2016}
Ashish Cherukuri, Enrique Mallada, and Jorge Cortés.
\newblock Asymptotic convergence of constrained primal–dual dynamics.
\newblock \emph{Systems \& Control Letters}, 87:\penalty0 10--15, January 2016.
\newblock ISSN 0167-6911.
\newblock \doi{10.1016/j.sysconle.2015.10.006}.
\newblock URL
  \url{https://www.sciencedirect.com/science/article/pii/S0167691115002078}.

\bibitem[Chewi(2024)]{chewi_log-concave_2024}
Sinho Chewi.
\newblock \emph{Log-Concave Sampling}.
\newblock March 2024.
\newblock URL \url{https://chewisinho.github.io/main.pdf}.

\bibitem[Chewi et~al.(2021)Chewi, Erdogdu, Li, Shen, and
  Zhang]{chewi_analysis_2021}
Sinho Chewi, Murat~A. Erdogdu, Mufan~Bill Li, Ruoqi Shen, and Matthew Zhang.
\newblock Analysis of {Langevin} {Monte} {Carlo} from {Poincar\'e} to
  {Log}-{Sobolev}, December 2021.
\newblock URL \url{http://arxiv.org/abs/2112.12662}.
\newblock arXiv:2112.12662 [math, stat].

\bibitem[Chiang et~al.(1987)Chiang, Hwang, and Sheu]{chiang_diffusion_1987}
Tzuu-Shuh Chiang, Chii-Ruey Hwang, and Shuenn~Jyi Sheu.
\newblock Diffusion for {Global} {Optimization} in $\mathbb{R}^n$.
\newblock \emph{SIAM Journal on Control and Optimization}, 25\penalty0
  (3):\penalty0 737--753, May 1987.
\newblock ISSN 0363-0129.
\newblock \doi{10.1137/0325042}.
\newblock URL \url{https://epubs.siam.org/doi/10.1137/0325042}.
\newblock Publisher: Society for Industrial and Applied Mathematics.

\bibitem[Dalalyan(2016)]{dalalyan_theoretical_2016}
Arnak~S. Dalalyan.
\newblock Theoretical guarantees for approximate sampling from smooth and
  log-concave densities, December 2016.
\newblock URL \url{http://arxiv.org/abs/1412.7392}.
\newblock arXiv:1412.7392 [math, stat].

\bibitem[Dambre et~al.(2012)Dambre, Verstraeten, Schrauwen, and
  Massar]{dambre_information_2012}
Joni Dambre, David Verstraeten, Benjamin Schrauwen, and Serge Massar.
\newblock Information {Processing} {Capacity} of {Dynamical} {Systems}.
\newblock \emph{Scientific Reports}, 2\penalty0 (1):\penalty0 514, July 2012.
\newblock ISSN 2045-2322.
\newblock \doi{10.1038/srep00514}.
\newblock URL \url{https://www.nature.com/articles/srep00514}.
\newblock Publisher: Nature Publishing Group.

\bibitem[Ding \& Li(2021)Ding and Li]{ding_langevin_2021}
Zhiyan Ding and Qin Li.
\newblock Langevin {Monte} {Carlo}: random coordinate descent and variance
  reduction.
\newblock \emph{The Journal of Machine Learning Research}, 22\penalty0
  (1):\penalty0 205:9312--205:9362, January 2021.
\newblock ISSN 1532-4435.

\bibitem[Ding et~al.(2020)Ding, Li, Lu, and Wright]{ding_random_2020}
Zhiyan Ding, Qin Li, Jianfeng Lu, and Stephen~J. Wright.
\newblock Random {Coordinate} {Underdamped} {Langevin} {Monte} {Carlo}, October
  2020.
\newblock URL \url{http://arxiv.org/abs/2010.11366}.
\newblock arXiv:2010.11366 [cs, stat].

\bibitem[Ding et~al.(2021)Ding, Li, Lu, and Wright]{ding_random_2021}
Zhiyan Ding, Qin Li, Jianfeng Lu, and Stephen~J. Wright.
\newblock Random {Coordinate} {Langevin} {Monte} {Carlo}.
\newblock In \emph{Proceedings of {Thirty} {Fourth} {Conference} on {Learning}
  {Theory}}, pp.\  1683--1710. PMLR, July 2021.
\newblock URL \url{https://proceedings.mlr.press/v134/ding21a.html}.
\newblock ISSN: 2640-3498.

\bibitem[Erementchouk et~al.(2022)Erementchouk, Shukla, and
  Mazumder]{erementchouk_computational_2022}
Mikhail Erementchouk, Aditya Shukla, and Pinaki Mazumder.
\newblock On computational capabilities of {Ising} machines based on nonlinear
  oscillators,.
\newblock \emph{Physica D: Nonlinear Phenomena}, 437:\penalty0 133334,
  September 2022.
\newblock ISSN 01672789.
\newblock \doi{10.1016/j.physd.2022.133334}.
\newblock URL \url{http://arxiv.org/abs/2105.07591}.
\newblock arXiv:2105.07591 [cond-mat, physics:physics].

\bibitem[Farghly \& Rebeschini(2021)Farghly and
  Rebeschini]{farghly_time-independent_2021}
Tyler Farghly and Patrick Rebeschini.
\newblock Time-independent {Generalization} {Bounds} for {SGLD} in {Non}-convex
  {Settings}, November 2021.
\newblock URL \url{http://arxiv.org/abs/2111.12876}.
\newblock arXiv:2111.12876 [cs, math, stat].

\bibitem[Hamerly et~al.(2019)Hamerly, Inagaki, McMahon, Venturelli, Marandi,
  Onodera, Ng, Langrock, Inaba, Honjo, Enbutsu, Umeki, Kasahara, Utsunomiya,
  Kako, Kawarabayashi, Byer, Fejer, Mabuchi, Englund, Rieffel, Takesue, and
  Yamamoto]{hamerly_experimental_2019}
Ryan Hamerly, Takahiro Inagaki, Peter~L. McMahon, Davide Venturelli, Alireza
  Marandi, Tatsuhiro Onodera, Edwin Ng, Carsten Langrock, Kensuke Inaba,
  Toshimori Honjo, Koji Enbutsu, Takeshi Umeki, Ryoichi Kasahara, Shoko
  Utsunomiya, Satoshi Kako, Ken-ichi Kawarabayashi, Robert~L. Byer, Martin~M.
  Fejer, Hideo Mabuchi, Dirk Englund, Eleanor Rieffel, Hiroki Takesue, and
  Yoshihisa Yamamoto.
\newblock Experimental investigation of performance differences between
  coherent {Ising} machines and a quantum annealer.
\newblock \emph{Science Advances}, 5\penalty0 (5):\penalty0 eaau0823, May 2019.
\newblock \doi{10.1126/sciadv.aau0823}.
\newblock URL \url{https://www.science.org/doi/full/10.1126/sciadv.aau0823}.
\newblock Publisher: American Association for the Advancement of Science.

\bibitem[He et~al.(2019)He, Lin, Ewetz, Yuan, and Fan]{he_noise_2019}
Zhezhi He, Jie Lin, Rickard Ewetz, Jiann-Shiun Yuan, and Deliang Fan.
\newblock Noise {Injection} {Adaption}: {End}-to-{End} {ReRAM} {Crossbar}
  {Non}-ideal {Effect} {Adaption} for {Neural} {Network} {Mapping}.
\newblock In \emph{Proceedings of the 56th {Annual} {Design} {Automation}
  {Conference} 2019}, {DAC} '19, pp.\  1--6, New York, NY, USA, June 2019.
  Association for Computing Machinery.
\newblock ISBN 978-1-4503-6725-7.
\newblock \doi{10.1145/3316781.3317870}.
\newblock URL \url{https://dl.acm.org/doi/10.1145/3316781.3317870}.

\bibitem[Hinton et~al.(2006)Hinton, Osindero, and Teh]{hinton2006fast}
Geoffrey~E Hinton, Simon Osindero, and Yee-Whye Teh.
\newblock A fast learning algorithm for deep belief nets.
\newblock \emph{Neural computation}, 18\penalty0 (7):\penalty0 1527--1554,
  2006.

\bibitem[Honjo et~al.(2021)Honjo, Sonobe, Inaba, Inagaki, Ikuta, Yamada,
  Kazama, Enbutsu, Umeki, Kasahara, Kawarabayashi, and
  Takesue]{honjo_100000-spin_2021}
Toshimori Honjo, Tomohiro Sonobe, Kensuke Inaba, Takahiro Inagaki, Takuya
  Ikuta, Yasuhiro Yamada, Takushi Kazama, Koji Enbutsu, Takeshi Umeki, Ryoichi
  Kasahara, Ken-ichi Kawarabayashi, and Hiroki Takesue.
\newblock 100,000-spin coherent {Ising} machine.
\newblock \emph{Science Advances}, 7\penalty0 (40):\penalty0 eabh0952,
  September 2021.
\newblock \doi{10.1126/sciadv.abh0952}.
\newblock URL \url{https://www.science.org/doi/full/10.1126/sciadv.abh0952}.
\newblock Publisher: American Association for the Advancement of Science.

\bibitem[Hu et~al.(2023)Hu, Angelatos, Khan, Vives, Türeci, Bello, Rowlands,
  Ribeill, and Türeci]{hu_tackling_2023}
Fangjun Hu, Gerasimos Angelatos, Saeed~A. Khan, Marti Vives, Esin Türeci, Leon
  Bello, Graham~E. Rowlands, Guilhem~J. Ribeill, and Hakan~E. Türeci.
\newblock Tackling {Sampling} {Noise} in {Physical} {Systems} for {Machine}
  {Learning} {Applications}: {Fundamental} {Limits} and {Eigentasks}.
\newblock \emph{Physical Review X}, 13\penalty0 (4):\penalty0 041020, October
  2023.
\newblock ISSN 2160-3308.
\newblock \doi{10.1103/PhysRevX.13.041020}.
\newblock URL \url{https://link.aps.org/doi/10.1103/PhysRevX.13.041020}.

\bibitem[Inaba et~al.(2023)Inaba, Yamada, and
  Takesue]{inaba_thermodynamic_2023}
Kensuke Inaba, Yasuhiro Yamada, and Hiroki Takesue.
\newblock Thermodynamic quantities of two-dimensional {Ising} models obtained
  by noisy mean field annealing and coherent {Ising} machine, February 2023.
\newblock URL \url{http://arxiv.org/abs/2302.01454}.
\newblock arXiv:2302.01454 [cond-mat, physics:quant-ph].

\bibitem[Inagaki et~al.(2016)Inagaki, Haribara, Igarashi, Sonobe, Tamate,
  Honjo, Marandi, McMahon, Umeki, Enbutsu, Tadanaga, Takenouchi, Aihara,
  Kawarabayashi, Inoue, Utsunomiya, and Takesue]{inagaki_coherent_2016}
Takahiro Inagaki, Yoshitaka Haribara, Koji Igarashi, Tomohiro Sonobe, Shuhei
  Tamate, Toshimori Honjo, Alireza Marandi, Peter~L. McMahon, Takeshi Umeki,
  Koji Enbutsu, Osamu Tadanaga, Hirokazu Takenouchi, Kazuyuki Aihara, Ken-ichi
  Kawarabayashi, Kyo Inoue, Shoko Utsunomiya, and Hiroki Takesue.
\newblock A coherent {Ising} machine for 2000-node optimization problems.
\newblock \emph{Science}, 354\penalty0 (6312):\penalty0 603--606, November
  2016.
\newblock \doi{10.1126/science.aah4243}.
\newblock URL \url{https://www.science.org/doi/full/10.1126/science.aah4243}.
\newblock Publisher: American Association for the Advancement of Science.

\bibitem[Jordan et~al.(1998)Jordan, Kinderlehrer, and
  Otto]{jordan_variational_1998}
Richard Jordan, David Kinderlehrer, and Felix Otto.
\newblock The {Variational} {Formulation} of the {Fokker}--{Planck} {Equation}.
\newblock \emph{SIAM Journal on Mathematical Analysis}, 29\penalty0
  (1):\penalty0 1--17, January 1998.
\newblock ISSN 0036-1410, 1095-7154.
\newblock \doi{10.1137/S0036141096303359}.
\newblock URL \url{http://epubs.siam.org/doi/10.1137/S0036141096303359}.

\bibitem[J. Earl \& W. Deem(2005)J. Earl and W. Deem]{jearl_parallel_2005}
David J. Earl and Michael W. Deem.
\newblock Parallel tempering: {Theory}, applications, and new perspectives.
\newblock \emph{Physical Chemistry Chemical Physics}, 7\penalty0 (23):\penalty0
  3910--3916, 2005.
\newblock \doi{10.1039/B509983H}.
\newblock URL
  \url{https://pubs.rsc.org/en/content/articlelanding/2005/cp/b509983h}.
\newblock Publisher: Royal Society of Chemistry.

\bibitem[Kirkpatrick(1984)]{kirkpatrick_optimization_1984}
Scott Kirkpatrick.
\newblock Optimization by simulated annealing: {Quantitative} studies.
\newblock \emph{Journal of Statistical Physics}, 34\penalty0 (5):\penalty0
  975--986, March 1984.
\newblock ISSN 1572-9613.
\newblock \doi{10.1007/BF01009452}.
\newblock URL \url{https://doi.org/10.1007/BF01009452}.

\bibitem[Klachko et~al.(2019)Klachko, Mahmoodi, and
  Strukov]{klachko_improving_2019}
Michael Klachko, Mohammad~Reza Mahmoodi, and Dmitri Strukov.
\newblock Improving {Noise} {Tolerance} of {Mixed}-{Signal} {Neural}
  {Networks}.
\newblock In \emph{2019 {International} {Joint} {Conference} on {Neural}
  {Networks} ({IJCNN})}, pp.\  1--8, July 2019.
\newblock \doi{10.1109/IJCNN.2019.8851966}.
\newblock URL \url{https://ieeexplore.ieee.org/abstract/document/8851966}.
\newblock ISSN: 2161-4407.

\bibitem[Lee(2022)]{lee_gibbs_2022}
Se~Yoon Lee.
\newblock Gibbs sampler and coordinate ascent variational inference: {A}
  set-theoretical review.
\newblock \emph{Communications in Statistics - Theory and Methods}, 51\penalty0
  (6):\penalty0 1549--1568, March 2022.
\newblock ISSN 0361-0926.
\newblock \doi{10.1080/03610926.2021.1921214}.
\newblock URL \url{https://doi.org/10.1080/03610926.2021.1921214}.
\newblock Publisher: Taylor \& Francis \_eprint:
  https://doi.org/10.1080/03610926.2021.1921214.

\bibitem[Li \& Wang(2022)Li and Wang]{li_sharp_2022}
Lei Li and Yuliang Wang.
\newblock A sharp uniform-in-time error estimate for {Stochastic} {Gradient}
  {Langevin} {Dynamics}, October 2022.
\newblock URL \url{http://arxiv.org/abs/2207.09304}.
\newblock arXiv:2207.09304 [cs, math, stat].

\bibitem[Long et~al.(2019)Long, She, and Mukhopadhyay]{long_design_2019}
Yun Long, Xueyuan She, and Saibal Mukhopadhyay.
\newblock Design of {Reliable} {DNN} {Accelerator} with {Un}-reliable {ReRAM}.
\newblock In \emph{2019 {Design}, {Automation} \& {Test} in {Europe}
  {Conference} \& {Exhibition} ({DATE})}, pp.\  1769--1774, March 2019.
\newblock \doi{10.23919/DATE.2019.8715178}.
\newblock URL \url{https://ieeexplore.ieee.org/abstract/document/8715178}.
\newblock ISSN: 1558-1101.

\bibitem[Lucas(2014)]{lucas_ising_2014}
Andrew Lucas.
\newblock Ising formulations of many {NP} problems.
\newblock \emph{Frontiers in Physics}, 2, 2014.
\newblock ISSN 2296-424X.
\newblock URL
  \url{https://www.frontiersin.org/articles/10.3389/fphy.2014.00005}.

\bibitem[Ma et~al.(2019)Ma, Chen, Jin, Flammarion, and
  Jordan]{ma_sampling_2019}
Yi-An Ma, Yuansi Chen, Chi Jin, Nicolas Flammarion, and Michael~I. Jordan.
\newblock Sampling can be faster than optimization.
\newblock \emph{Proceedings of the National Academy of Sciences}, 116\penalty0
  (42):\penalty0 20881--20885, October 2019.
\newblock \doi{10.1073/pnas.1820003116}.
\newblock URL \url{https://www.pnas.org/doi/abs/10.1073/pnas.1820003116}.
\newblock Publisher: Proceedings of the National Academy of Sciences.

\bibitem[Melanson et~al.(2023)Melanson, Khater, Aifer, Donatella, Gordon, Ahle,
  Crooks, Martinez, Sbahi, and Coles]{melanson_thermodynamic_2023}
Denis Melanson, Mohammad~Abu Khater, Maxwell Aifer, Kaelan Donatella,
  Max~Hunter Gordon, Thomas Ahle, Gavin Crooks, Antonio~J. Martinez, Faris
  Sbahi, and Patrick~J. Coles.
\newblock Thermodynamic {Computing} {System} for {AI} {Applications}, December
  2023.
\newblock URL \url{http://arxiv.org/abs/2312.04836}.
\newblock arXiv:2312.04836 [cond-mat].

\bibitem[Millidge et~al.(2022)Millidge, Salvatori, Song, Bogacz, and
  Lukasiewicz]{millidge_predictive_2022}
Beren Millidge, Tommaso Salvatori, Yuhang Song, Rafal Bogacz, and Thomas
  Lukasiewicz.
\newblock Predictive {Coding}: {Towards} a {Future} of {Deep} {Learning} beyond
  {Backpropagation}?, February 2022.
\newblock URL \url{http://arxiv.org/abs/2202.09467}.
\newblock arXiv:2202.09467 [cs].

\bibitem[Mohseni et~al.(2022)Mohseni, McMahon, and Byrnes]{mohseni_ising_2022}
Naeimeh Mohseni, Peter~L. McMahon, and Tim Byrnes.
\newblock Ising machines as hardware solvers of combinatorial optimization
  problems.
\newblock \emph{Nature Reviews Physics}, 4\penalty0 (6):\penalty0 363--379,
  June 2022.
\newblock ISSN 2522-5820.
\newblock \doi{10.1038/s42254-022-00440-8}.
\newblock URL \url{https://www.nature.com/articles/s42254-022-00440-8}.
\newblock Number: 6 Publisher: Nature Publishing Group.

\bibitem[Nesterov(2012)]{nesterov_efficiency_2012}
Yu. Nesterov.
\newblock Efficiency of {Coordinate} {Descent} {Methods} on {Huge}-{Scale}
  {Optimization} {Problems}.
\newblock \emph{SIAM Journal on Optimization}, 22\penalty0 (2):\penalty0
  341--362, January 2012.
\newblock ISSN 1052-6234.
\newblock \doi{10.1137/100802001}.
\newblock URL \url{https://epubs.siam.org/doi/abs/10.1137/100802001}.
\newblock Publisher: Society for Industrial and Applied Mathematics.

\bibitem[Ng et~al.(2022)Ng, Onodera, Kako, McMahon, Mabuchi, and
  Yamamoto]{ng_efficient_2022}
Edwin Ng, Tatsuhiro Onodera, Satoshi Kako, Peter~L. McMahon, Hideo Mabuchi, and
  Yoshihisa Yamamoto.
\newblock Efficient sampling of ground and low-energy {Ising} spin
  configurations with a coherent {Ising} machine.
\newblock \emph{Physical Review Research}, 4\penalty0 (1):\penalty0 013009,
  January 2022.
\newblock ISSN 2643-1564.
\newblock \doi{10.1103/PhysRevResearch.4.013009}.
\newblock URL \url{https://link.aps.org/doi/10.1103/PhysRevResearch.4.013009}.

\bibitem[Otto \& Villani(2000)Otto and Villani]{otto_generalization_2000}
F.~Otto and C.~Villani.
\newblock Generalization of an {Inequality} by {Talagrand} and {Links} with the
  {Logarithmic} {Sobolev} {Inequality}.
\newblock \emph{Journal of Functional Analysis}, 173\penalty0 (2):\penalty0
  361--400, June 2000.
\newblock ISSN 0022-1236.
\newblock \doi{10.1006/jfan.1999.3557}.
\newblock URL
  \url{https://www.sciencedirect.com/science/article/pii/S0022123699935577}.

\bibitem[Polyanskiy \& Wu(2016)Polyanskiy and Wu]{polyanskiy_wasserstein_2016}
Yury Polyanskiy and Yihong Wu.
\newblock Wasserstein {Continuity} of {Entropy} and {Outer} {Bounds} for
  {Interference} {Channels}.
\newblock \emph{IEEE Trans. Inf. Theor.}, 62\penalty0 (7):\penalty0 3992--4002,
  July 2016.
\newblock ISSN 0018-9448.
\newblock \doi{10.1109/TIT.2016.2562630}.
\newblock URL \url{https://doi.org/10.1109/TIT.2016.2562630}.

\bibitem[Pramanik et~al.(2023)Pramanik, Chatterjee, and
  Oza]{pramanik_convergence_2023}
Sayantan Pramanik, Sourav Chatterjee, and Harshkumar Oza.
\newblock Convergence {Analysis} of {Opto}-{Electronic} {Oscillator} based
  {Coherent} {Ising} {Machines}, December 2023.
\newblock URL \url{http://arxiv.org/abs/2312.04290}.
\newblock arXiv:2312.04290 [quant-ph].

\bibitem[Raginsky et~al.(2017)Raginsky, Rakhlin, and
  Telgarsky]{raginsky_non-convex_2017}
Maxim Raginsky, Alexander Rakhlin, and Matus Telgarsky.
\newblock Non-convex learning via {Stochastic} {Gradient} {Langevin}
  {Dynamics}: a nonasymptotic analysis.
\newblock In \emph{Proceedings of the 2017 {Conference} on {Learning}
  {Theory}}, pp.\  1674--1703. PMLR, June 2017.
\newblock URL \url{https://proceedings.mlr.press/v65/raginsky17a.html}.
\newblock ISSN: 2640-3498.

\bibitem[Raymond et~al.(2023)Raymond, Stevanovic, Bernoudy, Boothby, McGeoch,
  Berkley, Farré, and King]{raymond_hybrid_2023}
Jack Raymond, Radomir Stevanovic, William Bernoudy, Kelly Boothby, Catherine
  McGeoch, Andrew~J. Berkley, Pau Farré, and Andrew~D. King.
\newblock Hybrid quantum annealing for larger-than-{QPU} lattice-structured
  problems.
\newblock \emph{ACM Transactions on Quantum Computing}, 4\penalty0
  (3):\penalty0 1--30, September 2023.
\newblock ISSN 2643-6809, 2643-6817.
\newblock \doi{10.1145/3579368}.
\newblock URL \url{http://arxiv.org/abs/2202.03044}.
\newblock arXiv:2202.03044 [quant-ph].

\bibitem[Scellier \& Bengio(2017)Scellier and
  Bengio]{scellier_equilibrium_2017}
Benjamin Scellier and Yoshua Bengio.
\newblock Equilibrium {Propagation}: {Bridging} the {Gap} between
  {Energy}-{Based} {Models} and {Backpropagation}.
\newblock \emph{Frontiers in Computational Neuroscience}, 11, May 2017.
\newblock ISSN 1662-5188.
\newblock \doi{10.3389/fncom.2017.00024}.
\newblock URL
  \url{https://www.frontiersin.org/articles/10.3389/fncom.2017.00024}.
\newblock Publisher: Frontiers.

\bibitem[Scellier et~al.(2023)Scellier, Ernoult, Kendall, and
  Kumar]{scellier_energy-based_2023}
Benjamin Scellier, Maxence Ernoult, Jack Kendall, and Suhas Kumar.
\newblock Energy-based learning algorithms for analog computing: a comparative
  study.
\newblock \emph{Advances in Neural Information Processing Systems},
  36:\penalty0 52705--52731, December 2023.
\newblock URL
  \url{https://proceedings.neurips.cc/paper_files/paper/2023/hash/a52b0d191b619477cc798d544f4f0e4b-Abstract-Conference.html}.

\bibitem[Sharma et~al.(2022)Sharma, Afoakwa, Ignjatovic, and
  Huang]{sharma_increasing_2022}
Anshujit Sharma, Richard Afoakwa, Zeljko Ignjatovic, and Michael Huang.
\newblock Increasing ising machine capacity with multi-chip architectures.
\newblock In \emph{Proceedings of the 49th {Annual} {International} {Symposium}
  on {Computer} {Architecture}}, {ISCA} '22, pp.\  508--521, New York, NY, USA,
  June 2022. Association for Computing Machinery.
\newblock ISBN 978-1-4503-8610-4.
\newblock \doi{10.1145/3470496.3527414}.
\newblock URL \url{https://dl.acm.org/doi/10.1145/3470496.3527414}.

\bibitem[Sharma et~al.(2023)Sharma, Burns, and Huang]{sharma_combining_2023}
Anshujit Sharma, Matthew Burns, and Michael~C. Huang.
\newblock Combining {Cubic} {Dynamical} {Solvers} with {Make}/{Break}
  {Heuristics} to {Solve} {SAT}.
\newblock In Meena Mahajan and Friedrich Slivovsky (eds.), \emph{26th
  {International} {Conference} on {Theory} and {Applications} of
  {Satisfiability} {Testing} ({SAT} 2023)}, volume 271 of \emph{Leibniz
  {International} {Proceedings} in {Informatics} ({LIPIcs})}, pp.\
  25:1--25:21, Dagstuhl, Germany, 2023. Schloss Dagstuhl – Leibniz-Zentrum
  für Informatik.
\newblock ISBN 978-3-95977-286-0.
\newblock \doi{10.4230/LIPIcs.SAT.2023.25}.
\newblock URL \url{https://drops.dagstuhl.de/opus/volltexte/2023/18487}.
\newblock ISSN: 1868-8969.

\bibitem[Song et~al.(2024)Song, Wu, Liu, Li, Huang, and Geng]{song_ds-gl_2024}
Ruibing Song, Chunshu Wu, Chuan Liu, Ang Li, Michael Huang, and Tong Geng.
\newblock {DS}-{GL}: {Advancing} {Graph} {Learning} via {Harnessing} the
  {Power} of {Nature} within {Dynamic} {Systems}.
\newblock Technical Report PNNL-SA-196761, Pacific Northwest National
  Laboratory (PNNL), Richland, WA (United States), August 2024.
\newblock URL \url{https://www.osti.gov/biblio/2426329}.

\bibitem[Stern et~al.(2021)Stern, Hexner, Rocks, and
  Liu]{stern_supervised_2021}
Menachem Stern, Daniel Hexner, Jason~W. Rocks, and Andrea~J. Liu.
\newblock Supervised {Learning} in {Physical} {Networks}: {From} {Machine}
  {Learning} to {Learning} {Machines}.
\newblock \emph{Physical Review X}, 11\penalty0 (2):\penalty0 021045, May 2021.
\newblock \doi{10.1103/PhysRevX.11.021045}.
\newblock URL \url{https://link.aps.org/doi/10.1103/PhysRevX.11.021045}.
\newblock Publisher: American Physical Society.

\bibitem[Tibshirani(2017)]{tibshirani_dykstra_2017}
Ryan~J Tibshirani.
\newblock Dykstra' s {Algorithm}, {ADMM}, and {Coordinate} {Descent}:
  {Connections}, {Insights}, and {Extensions}.
\newblock In \emph{Advances in {Neural} {Information} {Processing} {Systems}},
  volume~30. Curran Associates, Inc., 2017.
\newblock URL
  \url{https://proceedings.neurips.cc/paper/2017/hash/5ef698cd9fe650923ea331c15af3b160-Abstract.html}.

\bibitem[Ushijima-Mwesigwa et~al.(2017)Ushijima-Mwesigwa, Negre, and
  Mniszewski]{ushijima-mwesigwa_graph_2017}
Hayato Ushijima-Mwesigwa, Christian F.~A. Negre, and Susan~M. Mniszewski.
\newblock Graph {Partitioning} using {Quantum} {Annealing} on the {D}-{Wave}
  {System}.
\newblock In \emph{Proceedings of the {Second} {International} {Workshop} on
  {Post} {Moores} {Era} {Supercomputing}}, {PMES}'17, pp.\  22--29, New York,
  NY, USA, November 2017. Association for Computing Machinery.
\newblock ISBN 978-1-4503-5126-3.
\newblock \doi{10.1145/3149526.3149531}.
\newblock URL \url{https://dl.acm.org/doi/10.1145/3149526.3149531}.

\bibitem[Vempala \& Wibisono(2019)Vempala and Wibisono]{vempala_rapid_2019}
Santosh Vempala and Andre Wibisono.
\newblock Rapid {Convergence} of the {Unadjusted} {Langevin} {Algorithm}:
  {Isoperimetry} {Suffices}.
\newblock In \emph{Advances in {Neural} {Information} {Processing} {Systems}},
  volume~32. Curran Associates, Inc., 2019.
\newblock URL
  \url{https://proceedings.neurips.cc/paper/2019/hash/65a99bb7a3115fdede20da98b08a370f-Abstract.html}.

\bibitem[Vengalam et~al.(2023)Vengalam, Liu, Geng, Wu, and
  Huang]{vengalam_supporting_2023}
Uday Kumar~Reddy Vengalam, Yongchao Liu, Tong Geng, Hui Wu, and Michael Huang.
\newblock {SUPPORTING} {ENERGY}-{BASED} {LEARNING} {WITH} {AN} {ISING}
  {MACHINE} {SUBSTRATE}: {A} {CASE} {STUDY} {ON} {RBM}.
\newblock In \emph{Proceedings of the 56th {Annual} {IEEE}/{ACM}
  {International} {Symposium} on {Microarchitecture}}, {MICRO} '23, pp.\
  465--478, New York, NY, USA, December 2023. Association for Computing
  Machinery.
\newblock ISBN 9798400703294.
\newblock \doi{10.1145/3613424.3614315}.
\newblock URL \url{https://dl.acm.org/doi/10.1145/3613424.3614315}.

\bibitem[Vono et~al.(2019)Vono, Dobigeon, and
  Chainais]{vono_split-and-augmented_2019}
Maxime Vono, Nicolas Dobigeon, and Pierre Chainais.
\newblock Split-and-{Augmented} {Gibbs} {Sampler}—{Application} to
  {Large}-{Scale} {Inference} {Problems}.
\newblock \emph{IEEE Transactions on Signal Processing}, 67\penalty0
  (6):\penalty0 1648--1661, March 2019.
\newblock ISSN 1941-0476.
\newblock \doi{10.1109/TSP.2019.2894825}.
\newblock URL \url{https://ieeexplore.ieee.org/abstract/document/8625467}.
\newblock Conference Name: IEEE Transactions on Signal Processing.

\bibitem[Vono et~al.(2022)Vono, Paulin, and Doucet]{vono_efficient_2022}
Maxime Vono, Daniel Paulin, and Arnaud Doucet.
\newblock Efficient {MCMC} sampling with dimension-free convergence rate using
  {ADMM}-type splitting.
\newblock \emph{The Journal of Machine Learning Research}, 23\penalty0
  (1):\penalty0 25:1100--25:1168, January 2022.
\newblock ISSN 1532-4435.

\bibitem[Wang \& Roychowdhury(2019)Wang and Roychowdhury]{wang_oim_2019}
Tianshi Wang and Jaijeet Roychowdhury.
\newblock {OIM}: {Oscillator}-based {Ising} {Machines} for {Solving}
  {Combinatorial} {Optimisation} {Problems}, March 2019.
\newblock URL \url{http://arxiv.org/abs/1903.07163}.
\newblock arXiv:1903.07163 [cs].

\bibitem[Wang et~al.(2013)Wang, Marandi, Wen, Byer, and
  Yamamoto]{wang_coherent_2013}
Zhe Wang, Alireza Marandi, Kai Wen, Robert~L. Byer, and Yoshihisa Yamamoto.
\newblock Coherent {Ising} machine based on degenerate optical parametric
  oscillators.
\newblock \emph{Physical Review A}, 88\penalty0 (6):\penalty0 063853, December
  2013.
\newblock \doi{10.1103/PhysRevA.88.063853}.
\newblock URL \url{https://link.aps.org/doi/10.1103/PhysRevA.88.063853}.
\newblock Publisher: American Physical Society.

\bibitem[Wibisono(2018)]{wibisono_sampling_2018}
Andre Wibisono.
\newblock Sampling as optimization in the space of measures: {The} {Langevin}
  dynamics as a composite optimization problem.
\newblock In \emph{Proceedings of the 31st {Conference} {On} {Learning}
  {Theory}}, pp.\  2093--3027. PMLR, July 2018.
\newblock URL \url{https://proceedings.mlr.press/v75/wibisono18a.html}.
\newblock ISSN: 2640-3498.

\bibitem[Wright(2015)]{wright_coordinate_2015}
Stephen~J. Wright.
\newblock Coordinate descent algorithms.
\newblock \emph{Mathematical Programming}, 151\penalty0 (1):\penalty0 3--34,
  June 2015.
\newblock ISSN 1436-4646.
\newblock \doi{10.1007/s10107-015-0892-3}.
\newblock URL \url{https://doi.org/10.1007/s10107-015-0892-3}.

\bibitem[Wu et~al.(2024)Wu, Song, Liu, Yang, Li, Huang, and Geng]{wuextending}
Chunshu Wu, Ruibing Song, Chuan Liu, Yunan Yang, Ang Li, Michael Huang, and
  Tong Geng.
\newblock Extending power of nature from binary to real-valued graph learning
  in real world.
\newblock In \emph{The Twelfth International Conference on Learning
  Representations}, 2024.

\bibitem[Xiao et~al.(2022)Xiao, Feinberg, Bennett, Prabhakar, Saxena, Agrawal,
  Agarwal, and Marinella]{xiao_accuracy_2022}
T.~Patrick Xiao, Ben Feinberg, Christopher~H. Bennett, Venkatraman Prabhakar,
  Prashant Saxena, Vineet Agrawal, Sapan Agarwal, and Matthew~J. Marinella.
\newblock On the {Accuracy} of {Analog} {Neural} {Network} {Inference}
  {Accelerators}.
\newblock \emph{IEEE Circuits and Systems Magazine}, 22\penalty0 (4):\penalty0
  26--48, 2022.
\newblock ISSN 1558-0830.
\newblock \doi{10.1109/MCAS.2022.3214409}.
\newblock URL
  \url{https://ieeexplore.ieee.org/document/10018023/?arnumber=10018023&tag=1}.
\newblock Conference Name: IEEE Circuits and Systems Magazine.

\bibitem[Xu et~al.(2020)Xu, Chen, Zou, and Gu]{xu_global_2020}
Pan Xu, Jinghui Chen, Difan Zou, and Quanquan Gu.
\newblock Global {Convergence} of {Langevin} {Dynamics} {Based} {Algorithms}
  for {Nonconvex} {Optimization}, October 2020.
\newblock URL \url{http://arxiv.org/abs/1707.06618}.
\newblock arXiv:1707.06618 [cs, math, stat].

\bibitem[Zhang et~al.(2023)Zhang, Akyildiz, Damoulas, and
  Sabanis]{zhang_nonasymptotic_2023}
Ying Zhang, Ömer~Deniz Akyildiz, Theodoros Damoulas, and Sotirios Sabanis.
\newblock Nonasymptotic {Estimates} for {Stochastic} {Gradient} {Langevin}
  {Dynamics} {Under} {Local} {Conditions} in {Nonconvex} {Optimization}.
\newblock \emph{Applied Mathematics \& Optimization}, 87\penalty0 (2):\penalty0
  25, January 2023.
\newblock ISSN 1432-0606.
\newblock \doi{10.1007/s00245-022-09932-6}.
\newblock URL \url{https://doi.org/10.1007/s00245-022-09932-6}.

\bibitem[Zhang et~al.(2022)Zhang, Vengalam, Sharma, Huang, and
  Ignjatovic]{zhang_qubrim_2022}
Yiqiao Zhang, Uday Kumar~Reddy Vengalam, Anshujit Sharma, Michael Huang, and
  Zeljko Ignjatovic.
\newblock {QuBRIM}: {A} {CMOS} {Compatible} {Resistively}-{Coupled} {Ising}
  {Machine} with {Quantized} {Nodal} {Interactions}.
\newblock In \emph{Proceedings of the 41st {IEEE}/{ACM} {International}
  {Conference} on {Computer}-{Aided} {Design}}, {ICCAD} '22, pp.\  1--8, New
  York, NY, USA, December 2022. Association for Computing Machinery.
\newblock ISBN 978-1-4503-9217-4.
\newblock \doi{10.1145/3508352.3549443}.
\newblock URL \url{https://dl.acm.org/doi/10.1145/3508352.3549443}.

\bibitem[Zhang et~al.(2017)Zhang, Liang, and Charikar]{zhang_hitting_2017}
Yuchen Zhang, Percy Liang, and Moses Charikar.
\newblock A {Hitting} {Time} {Analysis} of {Stochastic} {Gradient} {Langevin}
  {Dynamics}.
\newblock In \emph{Proceedings of the 2017 {Conference} on {Learning}
  {Theory}}, pp.\  1980--2022. PMLR, June 2017.
\newblock URL \url{https://proceedings.mlr.press/v65/zhang17b.html}.
\newblock ISSN: 2640-3498.

\end{thebibliography}

\vfill
\pagebreak
\appendix
\section*{Supplementary Materials}
Here we provide proofs and explanations of experimental methods. Additionally, we apply our analysis to bound $\kl$ for discrete-time variants of RBLD and CBLD (RBLMC and CBLMC). For clarity of notation, we omit the $\beta$ subscript when referring to the target distribution $\pi\triangleq\pi_\beta$.
\section{Experimental Methods}\label{sec:experimental_methods}
\subsection{Diffusion Simulation}
The purpose of our experiments is a ``proof-of-concept'', hence we focus on demonstrating \ding{172} the expected performance equivalence between RBLD and CBLD, \ding{173} the significance of stepsize choice, and \ding{174} the impact of stochastic perturbations.

We simulate Langevin SDEs using an Euler-Maruyama integration scheme with a time step size of \SI{1.6e-11}{} seconds. Block diffusions are simulated for a fixed number of steps, then the block is switched either cyclically or randomly, depending on the algorithm. All code is implemented in PyTorch and was run on a desktop system using an i9-13900k with 64 GB of RAM and an RTX 4090 GPU. Our simulator, plotting code, and data is publicly available at \url{https://github.com/ur-acal/BlockLangevin}.

We take the resistively coupled BRIM architecture from~\cite{afoakwa_brim_2021} with the Langevin perturbations proposed by~\cite{sharma_combining_2023} as our baseline DX. The BRIM architecture is more easily extensible to general classes of real-valued functions~\citep{sharma_combining_2023,song_ds-gl_2024,wuextending} than oscillator-based DXs~\citep{wang_oim_2019,inagaki_coherent_2016}, motivating the selection. 

We model the device using $\SI{310}{k Ohm}$ resistors and $\SI{50}{fF}$ capacitors, leading to an RC time constant of $\SI{1.55e-8}{s}$ and an effective step size of $\SI{1.55e-11}{s}$, which we use to plot total estimated DX time. These circuit parameters are comparable to those proposed in literature~\citep{afoakwa_brim_2021,zhang_qubrim_2022}, however different device parameters will simply rescale the x-axis.

\subsection{Target Potential}
As stated in the main text, we choose a Gaussian target measure to obtain a direct estimate of convergence rather than using proxy statistical observables, as done in~\cite{ding_langevin_2021}. The $d=50$ Gaussian used to produce Fig.~\ref{fig:numerical} was generated using the following procedure:
\begin{enumerate}
    \item Generate a $50\times 50$ matrix $\Sigma_\pi^{-1}$ with elements $\sim \textrm{Unif}[-5,5]$
    \item Make the matrix symmetric by setting $\Sigma_\pi^{-1}=\frac{1}{2}(\Sigma_\pi^{-1} + (\Sigma_\pi^{-1})^\top)$
    \item To make $\Sigma_\pi$ positive definite, set $\Sigma_\pi^{-1}=\Sigma^{-1}+1.2\lambda_{\min}I_{50}$
\end{enumerate}
The resulting matrix is symmetric and positive-definite, making it a valid similarity matrix. We then invert $\Sigma_\pi^{-1}$ to obtain the target covariance matrix $\Sigma_\pi$. We choose $[-5,5]$ as the distribution to test a larger range of perturbation strengths $\Delta\in[0.1,0.4]$, as the $W_2$ diverged much earlier ($\Delta<0.2$) with a uniform $[-1,1]$ distribution.

As our focus is sampling rather than optimization, we set $\beta=1$ for simplicity. We also assume the Gaussian mean is zero, making the target distribution
\begin{equation}
    \pi(x)\propto e^{-\frac{1}{2}x^\top\Sigma_\pi^{-1}x}.
\end{equation}
\subsection{Sampling Procedure}
We randomly initialize $N=10^4$ states and evolve them in parallel with equivalent block selections. Every 30 iterations (time $t$) we compute the empirical covariance matrix for time $\Sigma_t$ and the empirical mean $u_t$ to obtain the estimated $\mu_{t,Est}=\mathcal{N}(u_t, \Sigma_t)$. We then compute the similarity to the target Gaussian. For completeness we compute both $W_2$ and $\kl$ using 
\begin{equation}\label{eqn:w2_gauss}
    W^2_2(\mu_{t,Est}, \mathcal{N}(0,\Sigma_\pi))=||\overline x_t||^2+\mathrm{Tr}\left[\Sigma_t+\Sigma_\pi-2\sqrt{\sqrt{\Sigma_\pi}\Sigma_t\sqrt{\Sigma_\pi}}\right]
\end{equation}
for $W_2$ and a PyTorch library function for $\kl$, which computes 
\begin{equation}
    \KL{\mu_{t,Est}}{\mathcal{N}(0,\Sigma_\pi)}=\frac{1}{2}\left[\log\frac{\det \Sigma_\pi}{\det \Sigma_t}-d + \mathrm{Tr}[\Sigma_\pi^{-1}\Sigma_t] + u_t^T\Sigma_\pi^{-1}u_t\right].
\end{equation}

\section{Review of Langevin Dynamics}
In this section we review the traditional proof of Equation~\eqref{eqn:ld_kl_conv} as presented by~\cite{vempala_rapid_2019,chewi_log-concave_2024}. Proofs in succeeding sections follow similar processes, making a brief review useful for establishing context.

Recall that the LD SDE is given by
\begin{equation}\label{eqn:appdx_ld}
    dx = -\nabla f(x)dt + \sqrt{2\beta^{-1}}dW_t.
\end{equation}

The Fokker-Planck equation (FPE) for Equation~\eqref{eqn:appdx_ld} is given by~\citep{vempala_rapid_2019, jordan_variational_1998}
\begin{equation}
    \frac{\partial \mu_t}{\partial t}= \beta^{-1} \nabla^2\cdot \mu_t + \nabla\cdot [\mu_t\nabla f(x)]
\end{equation}
where $\mu_t$ is the law of $x(t)$. For convenience, we will abbreviate $\frac{\partial \mu_t}{\partial t}$ as $\partial_t \mu_t$.

Note that the right hand side of the FPE can be equivalently expressed as
\begin{equation}\label{eqn:log_fpe}
    \begin{split}
        \beta^{-1} \nabla^2\cdot \mu_t + \nabla\cdot [\mu_t\nabla f(x)]\\
        =\beta^{-1}\nabla\cdot [\nabla\mu_t + \beta\mu_t\nabla f(x)]\\
        =\beta^{-1}\nabla\cdot [\mu_t\nabla\ln\mu_t -\mu_t\nabla \ln\pi_\beta]\\
        =\beta^{-1}\nabla\cdot [\mu_t\nabla\ln\frac{\mu_t}{\pi_\beta}].
    \end{split}
\end{equation}

The time derivative of the KL-divergence is given by
\begin{equation}
    \begin{split}
        \partial_t\KL{\mu_t}{\pi_\beta}=\partial_t\int \mu_t(x) \ln\frac{\mu_t(x)}{\pi_\beta(x)}dx\\
        =\int \left[[\partial_t\mu_t(x)] \ln\frac{\mu_t(x)}{\pi_\beta(x)} + \mu_t(x) \partial_t\ln\frac{\mu_t(x)}{\pi_\beta(x)}\right] dx
    \end{split}
\end{equation}
The second term is equal to zero, since
\begin{equation}
    \begin{split}
       \int \mu_t(x) \partial_t\ln\frac{\mu_t(x)}{\pi_\beta(x)}dx = \int\partial_t \mu_t(x) dx=\partial_t\int \mu_t(x) dx=0.
    \end{split}
\end{equation}
Then, applying Equation~\eqref{eqn:log_fpe} and integrating by parts
\begin{equation}\label{eqn:ld_fi_relation}
\begin{split}
    \partial_t\KL{\mu_t}{\pi_\beta}=\int[\partial_t\mu_t(x)] \ln\frac{\mu_t(x)}{\pi_\beta(x)}dx\\
    =\int\beta^{-1}\nabla\cdot [\mu_t(x)\nabla\ln\frac{\mu_t(x)}{\pi_\beta(x)}] \ln\frac{\mu_t(x)}{\pi_\beta(x)}dx\\
    =-\int\beta^{-1}\mu_t\inner{\nabla\ln\frac{\mu_t(x)}{\pi_\beta(x)}}{\nabla\ln\frac{\mu_t(x)}{\pi_\beta(x)}dx}\\
    =-\int\beta^{-1}\mu_t\norm{\nabla\ln\frac{\mu_t(x)}{\pi_\beta(x)}}^2dx
    =-\beta^{-1}FI(\mu_t\Vert \pi_\beta)
\end{split}
\end{equation}
where $FI(\mu_t\Vert \pi_\beta)$ is the relative Fisher information of $\mu_t$ relative to $\pi_\beta$. 

If $\pi_\beta(x)\propto e^{-\beta f(x)}$ satisfies a log-Sobolev inequality then
\begin{equation}
    \KL{\mu_t}{\pi_\beta}\leq \frac{1}{2\gamma}FI(\mu_t\Vert \pi_\beta).
\end{equation}
Combining the LSI with Equation~\ref{eqn:ld_fi_relation}, we obtain
\begin{equation}
    \partial_t\KL{\mu_t}{\pi_\beta}=-\beta^{-1}FI(\mu_t\Vert \pi_\beta)\leq -2\gamma\beta^{-1}\KL{\mu_t}{\pi_\beta}
\end{equation}
which implies exponential convergence of $\KL{\mu_t}{\pi_\beta}$
\begin{equation}
    \KL{\mu_t}{\pi_\beta} =e^{-2\gamma\beta^{-1}t}\KL{\mu_0}{\pi_\beta}.
\end{equation}

\section{Randomized Block Langevin Diffusion (RBLD)}\label{sec:rbld_proof}
    In this section we provide proofs relating to Randomized Block Langevin Diffusion (RBLD, the focus of the main text) and a time-discretized version, Randomized Block Langevin Monte Carlo (RBLMC). RBLMC was previously introduced in~\cite{ding_random_2021} as a coordinate-wise scheme, however we examine block partitions. Moreover, our results using $\gamma$-LSI target measures are more general than the strongly log-concave convergence results given in that work. For simplicity, we use the shorthand $x_t\triangleq x(t)$ throughout.

Algorithm~\ref{alg:rbld} gives the structure of RBLD/RBLMC sampling, where $\phi=\{\phi_1,...,\phi_b\}$ is a discrete probability mass function over coordinate block indices.

\subsection{Continuous Time Iteration}
\begin{algorithm}
\caption{Randomized Block Langevin Dynamics (RBLD)}
\label{alg:rbld}
    \begin{algorithmic}[1]
        \Procedure{RBLD}{$x_0\in\textrm{dom}(f)$, Block Distribution $\phi$ over $\{B_1,...,B_b\}$, Step Size Set $\lambda\in\mathbb{R}^b_+$}
        \For{$k\geq 0$}
            \State Choose $i\sim\phi$ and set $t_{k+1}=t_k+\lambda_i$
            \State Sample:
            \begin{align}
                x_{t_{k+1}}&=x_{t_{k}}-\int_{t_k}^{t_{k+1}}U_i\nabla f(x)dt +\int_{t_k}^{t_{k+1}}\:U_i\sqrt{2\beta^{-1}}dW_t
            \end{align}
        \EndFor
        \EndProcedure
    \end{algorithmic}
\end{algorithm}
We first consider the case when each diffusion occurs in continuous time.
For a single iteration, we can formulate the evolution of the system by the following It\^{o} SDE:
\begin{equation}
    dx=-U_k\nabla \left(f(x)dt +\sqrt{2\beta^{-1}}dW_t\right)
\end{equation}

To prove continuous-time descent in KL-divergence, we combine standard Langevin gradient flow arguments with methodology inspired by Ref.~\cite{vempala_rapid_2019} when considering expectation terms.

\subsection{Fokker-Planck Equation}\label{ssec:fpe_rcld}
\newcommand{\filtermu}{\mu_{t|0}}

Let $\mu_t$ be the law of $x_t$, and let $\filtermu$ be the measure jointly conditioned \ding{172} on the state at time 0 and \ding{173} the choice of block $B_{k}$. Within a single step, $\filtermu$ will obey the Fokker-Planck continuity equation
\[\partial_t\filtermu=\mathrm{Tr}[U_k\beta^{-1}\nabla^2\filtermu]+\nabla\cdot(\filtermu U_k\nabla f(x_t)).\]

If we were tracking the diffusion over a single block, we would take expectation over the starting state $x_0$ while conditioning on the block index. However, as discussed in the main text, we take a ``meta-Eulerian'' perspective. Instead of tracking one block diffusion, our approach finds the average behavior of an ensemble of diffusion processes, each independently sampling their blocks according to $\phi$. We therefore take expectation over both $x_0$ and $B_{k}$ to derive the change in the ``ensemble'' measure $\mu_t$.

Therefore we have
\begin{equation}
    \begin{split}
\partial_t\mu_t&=Tr[\beta^{-1}U_{\phi}\nabla^2\mu_t]+\E[\nabla\cdot(\filtermu U_k\nabla f(x))].
    \end{split}
\end{equation}

Where we have defined $U_\phi\triangleq (\phi_1U_1,...,\phi_bU_b)\in\mathbb{R}^{d\times d}$.

Let $\nu$ be the joint law of $(x_0, B_{k})$. Note that
\begin{equation}
    \begin{split}
        \mu_t(x_t|x_0,  B_{k})\nu(x_0, B_{k})&=\mu_t(x_t)\nu(x_0,  B_{k}|x_t)\\
        &=\mu_t(x_t)\nu(x_0 | x_t,  B_{k})\nu( B_{k} | x) \\
        &=\mu_t(x_t)\nu(x_0 | x_t,  B_{k})\nu( B_{k}) \\
        &=\mu_t(x_t)\nu(x_0 | x_t,  B_{k})\phi_k .
    \end{split}
\end{equation}
Then we can express the second term as
\begin{equation}
    \begin{split}
        \E[\nabla\cdot(\filtermu U_k\nabla  f(x_t))]&=\nabla\cdot(\sum_{i=1}^b\int \mu_t(x_t|x_0,i)U_k\nabla  f(x_t)\nu(x_0,i)dx_0)\\
        &=\nabla\cdot(\sum_{i=1}^b\phi_i\int \mu_t(x_t)U_i\nabla  f(x_t)\nu(x_0|x_t, i)dx_0)\\
        &=\nabla\cdot(\mu_t(x_t)U_\phi\nabla f(x_t)) 
    \end{split}
\end{equation}
since
\[\sum_{i=1}^b\phi_i U_i\nabla  f(x_t)=U_\phi\nabla f(x_t).\]
Therefore, the FPE of the ``meta-Eulerian'' RBLD process is
\begin{equation}
\label{eqn:rbld_fpe_2}
\partial_t\mu_t=\mathrm{Tr}[\beta^{-1}U_\phi\nabla^2\mu_t]+\nabla\cdot(\mu_t U_\phi\nabla f(x_t)).
\end{equation}
Note that the we can use the identity $\nabla f(x)=\beta^{-1}\nabla \log \pi$ to re-express the FPE as
\begin{equation}
\label{eqn:rbld_fpe}
\partial_t\mu_t=\diverg\left(\beta^{-1}\mu_t U_\phi\nabla \log\frac{\mu_t}{\pi}\right).
\end{equation}
\subsection{KL-Divergence Contraction}\label{ssec:debrujin_rcld}
\begin{lemma}
    \[\KL{\mu_t}{\pi}\leq \KL{\mu_0}{\pi}e^{-2\beta^{-1}\gamma \lambda_{\min}\phi_{\min}}.\]
\end{lemma}
\begin{proof}
    
The proof follows conventional analyses of Langevin diffusion processes, e.g., see ~\cite{vempala_rapid_2019, chewi_log-concave_2024,chewi_analysis_2021}. However, we complete the proof anew for completeness, as well as to show the differences with baseline LD. 

With the time evolution of the measure, we can now express the time evolution of the KL-divergence
\begin{equation}
\begin{split}
    \partial_t\KL{\mu_t}{\pi}&=\partial_t\int\mu_t(x)\log\frac{\mu_t(x)}{\pi(x)}dx\\
    &=\int\partial_t[\mu_t(x)\log\frac{\mu_t(x)}{\pi(x)}]dx\\
    &=-\int\partial_t\mu_t(x)\log\frac{\mu_t(x)}{\pi(x)}dx + \overbrace{\int\partial_t\mu_t(x)dx}^{=0}\\
\end{split}
\end{equation}
where the second term is equal to zero since \[\int\partial_t\mu_t(x)dx=\partial_t\int\mu_t(x)dx=\partial_t[1]=0.\]
Using Eqn.~\eqref{eqn:rbld_fpe}, we then have
\begin{equation}
\begin{split}
    \partial_t\KL{\mu_t}{\pi}
    &=\int\{\diverg\left(\beta^{-1}\mu_t U_\phi\nabla \log\frac{\mu_t}{\pi}\right)\}\log\frac{\mu_t(x)}{\pi(x)}dx.\\
\end{split}
\end{equation}
Through integration by parts, we obtain
\begin{equation}
\begin{split}
    \partial_t\KL{\mu_t}{\pi}&=-\beta^{-1}\int\inner{U_\phi\mu_t\nabla\log\frac{\mu_t}{\pi}}{\nabla\log\frac{\mu_t(x)}{\pi(x)}}dx\\
    &=-\beta^{-1}\E_{\mu_t}\left[\inner{U_\phi\nabla\log\frac{\mu_t}{\pi}}{\nabla\log\frac{\mu_t(x)}{\pi(x)}}\right].\\
\end{split}
\end{equation}
$U_\phi$ is positive-definite with minimum eigenvalue $\phi_{\min}$, therefore
\begin{equation}
    \begin{split}
        \partial_t\KL{\mu_t}{\pi}=&-\beta^{-1}\E_{\mu_t}\left[\inner{U_\phi\nabla\log\frac{\mu_t}{\pi}}{\nabla\log\frac{\mu_t(x)}{\pi(x)}}\right]
        \leq-\beta^{-1}\phi_{\min}\E_{\mu_t}\brackets{\left|\left|\nabla\log\frac{\mu_t}{\pi}\right|\right|^2}\\
        =&-\beta^{-1}\phi_{\min}FI(\mu_t\Vert \pi)\leq -2\beta^{-1}\gamma\phi_{\min}\KL{\mu_t}{\pi}
    \end{split}
\end{equation}
where the last inequality utilizes the $\gamma$-LSI. Here we highlight a principle difference between LD and RBLD analysis. In LD, we have the ``de Brujin \textit{identity}''
\begin{equation}
    \partial_t\KL{\mu_t}{\pi}=-2\beta^{-1}\gamma FI(\mu_t\Vert \pi).
\end{equation}
However, for RBLD we have a ``de Brujin \textit{in}equality''
\begin{equation}
    \partial_t\KL{\mu_t}{\pi}\leq-\beta^{-1}\gamma\phi_{\min}FI(\mu_t\Vert \pi).
\end{equation}

We now integrate up to $\lambda_{k}$. Since this step size depends on the choice of $k$, we take expectation of $\KL{\mu_{\lambda_{i}}}{\pi}$ where $t_k=\sum^{k}_{i=1}\lambda_{k}$
\begin{equation}
    \E[\KL{\mu_k}{\pi}]\leq\E[e^{-2\gamma\beta^{-1}\phi_{\min}\lambda_i}]\KL{\mu_{k-1}}{\pi}
\end{equation}
or deterministically
\begin{equation}
    \KL{\mu_k}{\pi}\leq e^{-2\gamma\beta^{-1}\phi_{\min}\lambda_{\min}}\KL{\mu_{k-1}}{\pi}.
\end{equation}
Expanding the inequality $k$ times yields the result. 
\end{proof}

\subsection{RCLMC: Euler-Maruyama Discretization}\label{sec:rblmc_proof}
We now extend our analysis to discrete-time Randomized Block Langevin Monte Carlo (RBLMC), shown in Algorithm.~\ref{alg:rblmc}.
\begin{algorithm}
\caption{Randomized Block Langevin Monte Carlo (RBLMC)}
\label{alg:rblmc}
    \begin{algorithmic}[1]
        \Procedure{RBLMC}{$x_0\in\textrm{dom}(f)$, Block Distribution $\phi$ over $\{B_1,...,B_b\}$, Step Size Set $\lambda\in\mathbb{R}^b_+$}
        \For{$k\geq 0$}
            \State Choose $i\sim\phi$, sample $\xi^k\sim\mathcal{N}(0, I^d)$
            \State Set:
            \begin{align}
                x^{k+1}&=x^{k}-\lambda_iU_i\nabla f(x^{k})+U_i\sqrt{2\lambda_i}\xi^k
            \end{align}
        \EndFor
        \EndProcedure
    \end{algorithmic}
\end{algorithm}
While the continuous-time diffusion can be implemented on dynamical hardware, digital applications require an error bound in the discrete-setting. The following derivation closely follows the methods of ~\cite{vempala_rapid_2019} by modeling the divergence of the discrete scheme from a continuous-time interpolation. To simplify constant terms, we take $\beta=1$ for this section.

We now consider the SDE
\begin{equation}
    dx=U_k[-\nabla f(x_0)dt+\sqrt{2}dW_t]
\end{equation}
where $x_0$ is the initial state. The SDE has the solution
\begin{equation}
    x_t=x_0+U_k[-\nabla f(x_0)t+\sqrt{2t}\xi_t]
\end{equation}
for $t\in[0, \lambda_n]$ and $\xi_t\sim\mathcal{N}(0,I^{d_i})$. Conditioned on the initial state $x_0$ and the choice of $i$, we have the FPE
\begin{equation}
    \begin{split}
\partial_t\mu_{t|k,x_0}=\nabla^2\cdot\mu_{t|k,x_0}+\nabla\cdot{\mu_{t|k,x_0}U_i\nabla  f(x_0)}.
    \end{split}
\end{equation}
Taking expectation over both sides (as previously) yields
\begin{equation}
    \begin{split}
\partial_t\mu=\mathrm{Tr}[U_\phi\nabla^2\mu_{t|k,x_0}]+\nabla\cdot{\E[\mu_{t|k,x_0}U_i\nabla  f(x_0)]}.
    \end{split}
\end{equation}
Again noting that the choice of block and the initial state $x_0$ are independent, we can express the expectation as
\begin{equation}
    \begin{split}
        \E[\mu_{t|k,x_0}U_k\nabla  f(x_0)]=&\sum_{i=1}^b\phi_i\int\mu(x_t|i,x_0)\nu(x_0)U_i\nabla  f(x_0)dx_0.
    \end{split}
\end{equation}
Note that while $x_0$ and $\phi_i$ are independent random variables, they are not independent when conditioned on $x_t$. We then have
\begin{equation}
    \begin{split}
        \phi_i\mu(x_t|i,x_0)\nu(x_0)&=\mu(x_t|i,x_0)\nu(x_0,i)\\
        &=\mu(x_t)\nu(x_0,i|x_t)\\
        &=\mu(x_t)\nu(x_0|x_t,i)\phi_{i|x_t}\\
        &=\mu(x_t)\nu(x_0|x_t,i)\phi_{i}.
    \end{split}
\end{equation}

Then
\begin{equation}
    \begin{split}
        \sum_{i=1}^b\phi_i\int\mu(x_t|i,x_0)\nu(x_0)U_i\nabla  f(x_0)dx_0&=\sum_{i=1}^b\phi_i\int\mu(x_t)\nu(x_0|x_t,i)U_i\nabla  f(x_0)dx_0\\
        &=\mu(x_t)\int\nu(x_0|x_t)U_\phi\nabla f(x_0)dx_0\\
        &=\mu_tU_\phi\E[\nabla f(x_0)].
    \end{split}
\end{equation}
We then have the following FPE
\begin{equation}
\begin{split}
    \partial_t\mu_t=\mathrm{Tr}[U_\phi\nabla^2\mu_t] + \nabla\cdot[\mu_tU_\phi\E[\nabla f(x_0)]]\\
    =\nabla\cdot[U_\phi\nabla\mu_t+\mu_tU_\phi\E[\nabla f(x_0)]].
\end{split}
\end{equation}

Combining our previous argument with the analysis of \cite{vempala_rapid_2019}, we have
\begin{equation}
 \begin{split}
    \partial_t\KL{\mu_t}{\pi} &=\partial_t\int\mu_t(x)\log\frac{\mu_t(x)}{\pi(x)}dx\\
    &=\int\partial_t\mu_t(x)\log\frac{\mu_t(x)}{\pi(x)}dx\\
    &=\int\nabla\cdot[U_\phi\nabla\mu_t+\mu_tU_\phi\E[\nabla f(x_0)]]\log\frac{\mu_t(x)}{\pi(x)}dx\\
    &=-\int\inner{U_\phi\nabla\mu_t+\mu_tU_\phi\E[\nabla f(x_0)]]}{\nabla\log\frac{\mu_t(x)}{\pi(x)}}dx\\
    &=-\int\inner{U_\phi\mu_t\nabla\log\mu_t+U_\phi\mu_t\nabla\log\pi-\mu_t\nabla\log\pi+\mu_tU_\phi\E[\nabla f(x_0)]]}{\nabla\log\frac{\mu_t(x)}{\pi(x)}}dx\\
    &=-\int\inner{U_\phi\mu_t\nabla\log\frac{\mu_t}{\pi}+\mu_tU_\phi\E[\nabla f(x_0)-\nabla f(x_t)]]}{\nabla\log\frac{\mu_t(x)}{\pi(x)}}dx\\
    &= -\E[\Vert U_\phi^{1/2}\nabla\log\frac{\mu_t(x)}{\pi(x)}\Vert ^2]+\E[\inner{U_\phi^{1/2}\E[\nabla f(x_t)-\nabla f(x_0)]}{U_\phi^{1/2}\nabla\log\frac{\mu_t(x)}{\pi(x)}}].\\
 \end{split}   
\end{equation}

where we have used the fact that $U_\phi$ is a diagonal matrix with non-negative entries, so  $U_\phi=U_\phi^{1/2}U_\phi^{1/2}=(U_\phi^{1/2})^TU_\phi^{1/2}$.
Then we have (by Cauchy-Schwartz and Young's)
\begin{equation}
    \begin{split}
        \E[\inner{U_\phi^{1/2}\E[\nabla f(x_t)-\nabla f(x_0)]}{U_\phi^{1/2}\nabla\log\frac{\mu_t(x)}{\pi(x)}}]&\leq \E[\Vert U_\phi^{1/2}\E[\nabla f(x_t)-\nabla f(x_0)]\Vert ^2]+\frac{1}{4}\E\Vert U_\phi^{1/2}\nabla\log\frac{\mu_t(x)}{\pi(x)}\Vert ^2\\
       &=\E[\Vert U_\phi^{1/2}[\nabla f(x_t)-\nabla f(x_0)]\Vert ^2]+\frac{1}{4}\E\Vert U_\phi^{1/2}\nabla\log\frac{\mu_t(x)}{\pi(x)}\Vert ^2.\\
    \end{split}
\end{equation}
We can decompose the first term as
\begin{equation}
\begin{split}
    \E[\Vert U_\phi^{1/2}[\nabla f(x_t)-\nabla f(x_0)]\Vert ^2]&=\sum_{i=1}^b\phi_i\Vert U_k\nabla  f(x_t)-U_k\nabla  f(x_0)\Vert ^2.\\
\end{split}
\end{equation}

In line with the presentation in the draft \cite{chewi_log-concave_2024} we apply Lemma 16 from \cite{chewi_analysis_2021}, which only requires smoothness and $L^2$ integrability in the marginal potential:
\begin{lemma}[Lemma 16 of \cite{chewi_analysis_2021}]
\label{lemma:chewi_grad}
    Assume probability measure $\pi\propto e^{-f(x)}\in\mathcal{P}^2(\R^d)$ has $L$-smooth potential $f$. Then for any probability measure $\mu$
    \begin{equation}
        \E_\mu[\Vert \nabla f\Vert ^2]\leq FI(\mu\Vert \pi) + 2dL.
    \end{equation}
\end{lemma}

By the smoothness of $f$, we have:
\begin{equation}
    \begin{split}
        \E\Vert U_k\nabla f(x_t)-U_k\nabla f(x_0)\Vert ^2 &\leq 2L_i^2\E\Vert x_t-x_0\Vert ^2=2L_i^2\E\Vert U_kt\nabla f(x_0) + U_k \sqrt{2}W_t\Vert ^2\\
        \leq& 2L_i^2t^2\E\Vert U_k\nabla f(x_0)+U_k\nabla f(x_t)-U_k\nabla f(x_t)\Vert ^2 + \E[2d_iL_i^2t]\\
        \leq& 2L_i^2t^2\E\Vert U_k\nabla f(x_0)-U_k\nabla f(x_t)\Vert ^2+2L_i^2\E\Vert U_k\nabla f(x_t)\Vert ^2 + \E[2d_iL_i^2t].\\
    \end{split}
\end{equation}
Suppose $t\leq \lambda_i\leq \frac{1}{2L_i}$, then
\begin{equation}
    \begin{split}
        \E\Vert U_k\nabla f(x_t)-U_k\nabla f(x_0)\Vert ^2 &\leq \frac{1}{2}\E\Vert U_k\nabla f(x_0)-U_k\nabla f(x_t)\Vert ^2+2L_i^2\E\Vert U_k\nabla f(x_t)\Vert ^2 + \E[2d_iL_i^2t].\\
    \end{split}
\end{equation}
Hence
\begin{equation}
    \begin{split}
        \E\Vert U_k\nabla f(x_t)-U_k\nabla f(x_0)\Vert ^2 &\leq 4L_i^2\E\Vert U_k\nabla f(x_t)\Vert ^2 + \E[4d_iL_i^2t].\\
    \end{split}
\end{equation}
Plugging in Lemma~\ref{lemma:chewi_grad} yields
\begin{equation}
    \begin{split}
        \E\Vert U_k\nabla f(x_t)-U_k\nabla f(x_0)\Vert ^2 &\leq 4L_i^2FI(\mu\Vert \pi) + \E[8td_iL_i^3 + 4d_iL_i^2t].\\
    \end{split}
\end{equation}

Assume $\lambda_i\leq\frac{\sqrt{\phi_{\min}}}{4L_i}$. Then
\begin{equation}
    \begin{split}
\E[\sum_{i=1}^b\phi_i [4t^2L_i^2FI(\mu_{B_i}\Vert \pi_{B_i}) + 8dL_i^3t^2+4d_iL_i^2t]]&\leq \E[\sum_{i=1}^b\frac{\phi_{\min}\phi_{i}}{4} FI(\mu_{B_i}\Vert \pi_{B_i}) + \sum_{i=1}^b\phi_i[8dL_i^3t^2+4d_iL_i^2t]] \\
&\leq FI(\mu_t\Vert \pi)\E[ \sum_{i=1}^b\phi_i \frac{\phi_{\min}}{4}+ \sum_{i=1}^b\phi_i[8d_iL_i^3t^2+4d_iL_i^2t]] \\
&=\frac{\phi_{\min}}{4}FI(\mu_t\Vert \pi)+ \E[\sum_{i=1}^b\phi_i[8d_iL_i^3t^2+4d_iL_i^2t]] \\
&\leq\frac{\phi_{\min}}{4}FI(\mu_t\Vert \pi)+ \E[6d_iL_i^2t]. \\
    \end{split}
\end{equation}
We then have
\begin{equation}
    \partial_t\KL{\mu_t}{\pi} \leq -\frac{\phi_{\min}}{2}FI(\mu_t\Vert \pi)+6\E[d_iL_i^2]t\leq -\phi_{\min}\gamma \KL{\mu_t}{\pi}+6\E[d_iL_i^2t].
\end{equation}

We start by multiplying both sides by $e^{-\phi_{\min}\gamma t}$ and integrating from $t=0$ to $\lambda_{i}$
\begin{equation}
    \begin{split}
        KL(\mu_{\lambda_{i}}\Vert \pi)\leq e^{-\phi_{\min}\lambda_{i}}\KL{\mu_0}{\pi}+3\E[d_iL_i^2\lambda_{i}^2].
    \end{split}
\end{equation}
Taking expectation over $i$ then gives the result
\begin{equation}
\label{rclmc:descent}
    \begin{split}
        \E_i[KL(\mu_{\lambda_{i}}\Vert \pi)]\leq \E_i[e^{-\gamma\phi_{\min}\lambda_{i}}]\KL{\mu_0}{\pi}+3\E[d_iL_i^2\lambda_i^2].
    \end{split}
\end{equation}

Iterating ~\ref{rclmc:descent} gives
\begin{equation}
\begin{split}
    \E[KL(\mu^k\Vert \pi)]\leq \E[e^{-\gamma\phi_{\min}\lambda_{\min}}]^k\KL{\mu_0}{\pi}+3\E[d_iL_i^2\lambda_{i}^2]\sum_{i=0}^k\E[e^{-\gamma\phi_{\min}\lambda_{\min}}]^i\\
    \leq e^{-\gamma\phi_{\min}\lambda_{i}k}\KL{\mu_0}{\pi}+\frac{4}{\gamma\phi_{\min}\lambda_{\min}}\E[d_iL_i^2\lambda_{i}^2],
\end{split}
\end{equation}
where we first bound using the minimum step size, then apply the power series bound
\[\sum_{i=0}^k\E[e^{-\gamma\phi_{\min}\lambda_{i}}]^i\leq\sum_{i=0}^ke^{-\gamma\phi_{\min}\lambda_{\min}}\leq\frac{1}{1-e^{-\gamma\phi_{\min}\lambda_{\min}}}\]
and then apply $\frac{1}{1-e^{-a}}\leq \frac{4}{3a}$ to obtain
\[\frac{1}{1-e^{-\gamma\phi_{\min}\lambda_{\min}}}\leq \frac{4}{3\gamma\phi_{\min}\lambda_{\min}}.\]

\section{Cyclic Block Langevin Diffusion}\label{sec:cyclic_appdx}
    In this section we provide proofs relating to Cyclic Block Langevin Diffusion (CBLD, the focus of the main text) and a time-discretized version, Cyclic Block Langevin Monte Carlo (CBLMC).
    
    The CBLD sampling algorithm is shown in Algorithm~\ref{alg:cbld}:
    \begin{algorithm}
    \caption{Cyclic Block Langevin Diffusion (CBLD)}
    \label{alg:cbld}
        \begin{algorithmic}[1]
            \Procedure{CBLD}{$x_0\in\textrm{dom}(f)$, Block Permutation $\sigma=\{B_{1},...,B_b\}$, Step Sizes $\lambda\in\mathbb{R}^b_+$}
            \For{$k\geq 0$}
                \State Define $\tau_0=t_{kb}$
                \For{$n=1$ to $b$}
                    \State Choose $i=\sigma_n$ and set $\tau_n=\tau_{n-1}+\lambda_i$
                    \State Sample:
            \begin{align}
                x_{\tau_{n}}&=x_{\tau_{n-1}}-\int_{\tau_{n-1}}^{\tau_n}U_i\nabla f(x)dt +\int_{\tau_{n-1}}^{\tau_n}U_i\sqrt{2\beta^{-1} \lambda_i}dW_t
            \end{align}
                \EndFor
                \State Set: $x_{t_{(k+1)b}}=x_{\tau_b}$
                \State \hspace{21px}$t_{(k+1)b}=\tau_b$
            \EndFor
            \EndProcedure
        \end{algorithmic}
    \end{algorithm}

A crucial identity used in our analysis is the ``chain lemma'' for KL-divergence. For any two distributions $\mu$
\begin{equation}
\KL{\mu_t}{\pi}=\E[\kl(\mu_{t|B}\Vert \pi_{|B})] + \kl(\mu_{t,B} | \pi_{B})
\end{equation}
where $B$ is a subspace of $\R^d$, 
$\kl(\mu_{t|B}\Vert \pi_{|B})$ 
is the KL-divergence of $\mu_t$ and $\pi$ conditioned on an element of $B$, and 
$\kl(\mu_{t,B} | \pi_{B})$ is the KL-divergence of 
$\mu_t$ and $\pi$ marginalized over $\R\setminus B$. 
We also state two trivial lemmas for any $\gamma$-LSI distribution $\nu$. We first state an equivalent definition of Assumption~\ref{assmp:lsi}.
\begin{definition}[Alternative LSI]
    $\pi\propto\exp[-\beta f(x)]$ satisfies a log-Sobolev inequality (LSI) with $C_{LSI}=\frac{1}{\gamma}$ if for all smooth $g$:
    \begin{equation}
        \E_\pi[g^2\log g^2]-\E_{\pi}[g^2]\log \E_{\pi}[g^2] \leq\frac{1}{2\gamma}\E_{\pi}[\Vert \nabla g\Vert ^2]
    \end{equation}
\end{definition}
where the equivalence with the previous statement follows by choosing $g^2(x)=\frac{\mu(x)}{\pi(x)}$.
\begin{lemma}
\label{lemma:lsi_marginal}
    Suppose $A,B$ are disjoint subspaces of $\R^d$ with $A\cup B=\R^d$. Then the $A$ marginal $\nu_{A}$ also satisfies $\gamma$-LSI.
\end{lemma}
\begin{proof}
    By the LSI, for any smooth $g:\R^d\to\R$
    \[\E_\nu\left[g^2\log g^2\right]-\E_\nu\left[g^2\right]\log\E_\nu\left[ g^2\right]\leq\E_\nu\left[\Vert \nabla g\Vert ^2\right].\]
    For $g:A\to\R$, we can re-express the terms as
    \[\E_{\nu_{B|A}}\E_{\nu_{A}}[\left[g^2\log g^2\right]-\E_{\nu_{B|A}}\left(\E_{\nu_{A}}\left[g^2\right]\log\E_\nu\left[ g^2\right]\right)\leq\E_{\nu_{B|A}}\E_{\nu_{A}}\left[\Vert \nabla g\Vert ^2\right].\]
    Since $\E_{\nu_{B|A}}[g(z)]=g(z)$ for all $z\in A$, we simplify to
    \[\E_{\nu_{A}}\left[g^2\log g^2\right]-\E_{\nu_{A}}\left[g^2\right]\log\E_\nu\left[ g^2\right]\leq\E_{\nu_{A}}\left[\Vert \nabla g\Vert ^2\right].\]
\end{proof}

Recall that the sub-step dynamics are described by the SDE
\begin{equation}
\label{eqn:sde_cbld}
    dx=U_n\brackets{-\nabla f(x)dt + \sqrt{2\beta}dW^{n}_t}.
\end{equation}
where $dW^{n}_t$ denotes $d_{n}$-dimensional Brownian noise.
We can then derive the coordinate Fokker-Planck equation: 
\begin{lemma}
\label{lemma:cond_fpe}
    Let $\mu_{t|x_0}$ be the law of $x$ at time $t\in[0,\lambda_n]$ described by the SDE in Equation~\eqref{eqn:sde_cbld}, where $\mu_{t|x_0}$ is conditioned on the starting state $x_0$. Then $\partial_t\mu_{t,\overline{B}_n|x_0}=0$ and \[\partial_t\mu_{t,B_n|\overline{B}_n,x_0}=\beta^{-1}\nabla^2\cdot\mu_{B_n|\overline{B}_n,x_0} + \diverg(\mu_{B_n|\overline{B}_n,x_0}\nabla f(x_t))\]
    is the Fokker-Planck equation for the subspace diffusion.
\end{lemma}
\begin{proof}
The second claim is trivially shown using It\^{o}'s Lemma. Note that since $\mu_{t,B_n|\overline{B}_n,x_0}$ is only supported on $B_n$:
\begin{enumerate}
    \item $\textrm{Tr}[\beta^{-1}U_n\nabla^2\mu_{t,B_n|\overline{B}_n,x_0}]=\beta^{-1}\nabla^2\cdot\mu_{t,B_n|\overline{B}_n,x_0}$
    \item $\diverg{\mu_{t,B_n|\overline{B}_n,x_0}\nabla f(x_t)}=\diverg{\mu_{t,B_n|\overline{B}_n,x_0}\nabla f(x_t)}$
\end{enumerate}
Then we have 
\[\partial_t\mu_{t,B_n|\overline{B}_n,x_0}=\beta^{-1}\nabla^2\cdot\mu_{B_n|\overline{B}_n,x_0} + \diverg(\mu_{B_n|\overline{B}_n,x_0}\nabla f(x_t)).\]
We now use this to prove the first claim.

Consider the law of $x$ in sub-step $n$ conditioned on the initial state $x_0$ given by $\mu_{t|x_0}$. Note that $\mu_{t|x_0}=\mu_{t,B_n|\overline{B}_n,x_0}\mu_{t,\overline{B}_n|x_0}$. 

By the Fokker-Planck equation associated with the SDE and the product rule, we have:
\begin{equation}
    \begin{split}
        \partial_t\mu_{t|x_0}=&\beta^{-1}\textrm{Tr}[U_n^T\nabla^2 \mu_{t|x_0}] + \diverg(\mu_{t|x_0}U_n\nabla f(x))\\
        \mu_{t,\overline{B}_n|x_0}\partial_t\mu_{t,B_n|\overline{B}_n,x_0}+\mu_{t,B_n|\overline{B}_n,x_0}\partial_t\mu_{t,\overline{B}_n|x_0}=&\beta^{-1}\textrm{Tr}[U_n^T\nabla^2 \mu_{t,B_n|\overline{B}_n,x_0}\mu_{t,\overline{B}_n|x_0}] \\
        &+ \diverg(\mu_{t,B_n|\overline{B}_n,x_0}\mu_{t,\overline{B}_n|x_0}U_n\nabla f(x)).\\
    \end{split}
\end{equation}
Note that
\begin{align}
    \beta^{-1}\textrm{Tr}[U_n^T\nabla^2 \mu_{t,B_n|\overline{B}_n,x_0}\mu_{t,\overline{B}_n|x_0}]&=\beta^{-1}\mu_{t,\overline{B}_n|x_0}\nabla^2\cdot \mu_{t,B_n|\overline{B}_n,x_0}\\
    \text{and }\\\diverg(\mu_{t,B_n|\overline{B}_n,x_0}\mu_{t,\overline{B}_n|x_0}U_n\nabla f(x))&=\mu_{t,\overline{B}_n|x_0}\diverg(\mu_{t,B_n|\overline{B}_n,x_0}\nabla f(x)).
\end{align}
We then have
\begin{align}
    \mu_{t,\overline{B}_n|x_0}(\overbrace{\partial_t\mu_{t,B_n|\overline{B}_n,x_0}-\beta^{-1}\nabla^2\cdot \mu_{t,B_n|\overline{B}_n,x_0} - \diverg(\mu_{t,B_n|\overline{B}_n,x_0}\nabla f(x))}^{\text{\large\ding{172}}})=\mu_{t,B_n|\overline{B}_n,x_0}\partial_t\mu_{t,\overline{B}_n|x_0}.
\end{align}
We assume that $\mu_{t}$ is supported on $\R^d$, therefore $\mu_{t|x_0}=\mu_{t,\overline{B}_n|x_0}\mu_{t,B_n|\overline{B}_n,x_0}>0$. 

As previously discussed, It\^{o}'s lemma implies \ding{172} is 0. For equality to hold, then, $\partial_t\mu_{t,\overline{B}_n}=0$. 
\end{proof}

We prove the following technical lemma for later use in the descent bound:
\begin{lemma}
\label{lemma:kl_cond}
    Suppose $A,B$ are disjoint subspaces of $\R^d$. Then we have
    \[\KL{\mu_{A}}{\pi_{A}}\leq \KL{\mu_{A|B}}{\pi_{A|B}}.\]
\end{lemma}
\begin{proof}
    Note that for all $x\in A$ 
    \[\mu_A(x)=\int_B\mu_{A,B}(x,y)dy=\int_B\mu_{A}(x|y)\mu_B(y)dy=\E_{y\in B}[\mu_{A}(x|y)]\triangleq\E_{B}[\mu_{A|B}].\]

    By the convexity of the KL-divergence and Jensen's Inequality
    \[\KL{\mu_{A}}{\pi_{A}}= KL(\E_{B}[\mu_{A|B}]\Vert \E_{B}[\pi_{A|B}])\leq \E_{B}[\KL{\mu_{A|B}}{\pi_{A|B}}]\triangleq \KL{\mu_{A|B}}{\pi_{A|B}}.\]
\end{proof}
Lemma~\ref{lemma:kl_cond} can be considered a restatement of the ``data processing inequality''. Removing the conditioning on subspace $B$ effectively reduces the available information, akin to a noisy channel, decreasing the divergence between distributions.

\begin{lemma}
\label{lemma:single_step_bcld}
    \[\KL{\mu_{n}}{\pi}\leq e^{-2\gamma\beta^{-1} \lambda_n}\brackets{\KL{\mu_{n-1}}{\pi}} + (1-e^{-2\gamma\beta^{-1} \lambda_n})\KL{\mu_{n-1,{\overline{B}_1}}}{\pi_{\overline{B}_1}}\]
\end{lemma}
\begin{proof}

Using Lemma~\ref{lemma:cond_fpe}, we can show by standard arguments~\cite{vempala_rapid_2019,chewi_analysis_2021} that within sub-step $n$:
\begin{equation}
    \label{eqn:cond_descent}
    \KL{\mu_{t,\overline{B}_1}}{\pi_{\overline{B}_1}}\leq
    e^{-2\gamma \beta^{-1}t}\KL{\mu_{0,\overline{B}_1}}{\pi_{\overline{B}_1}}
\end{equation}
Using ~\eqref{eqn:cond_descent} and the chain rule for KL-divergence
\begin{equation}
    \begin{split}
\KL{\mu_n}{\pi}&=\E\brackets{\KL{\mu_{n,B_1|\overline{B}_1}}{\pi_{B_1|\overline{B}_1}}} + \KL{\mu_{n-1,\overline{B}_1}}{\pi_{\overline{B}_1}}\\
    &\leq e^{-2\gamma \lambda_n\beta^{-1}}\E\brackets{\KL{\mu_{n-1,B_1|\overline{B}_1}}{\pi_{B_1|\overline{B}_1}}} + \KL{\mu_{n-1,\overline{B}_1}}{\pi_{\overline{B}_1}}\\
    &= e^{-2\gamma \lambda_n\beta^{-1}}\brackets{\KL{\mu_{n-1,B_1|\overline{B}_1} }{\pi_{B_1|\overline{B}_1}} - \KL{\mu_{n-1,{\overline{B}_1}}}{\pi_{\overline{B}_1})}} + \KL{\mu_{n-1,{\overline{B}_1}}}{\pi_{\overline{B}_1}}\\
    &= e^{-2\gamma \lambda_n\beta^{-1}}\KL{\mu_{n-1} }{\pi} + (1-e^{-2\gamma \lambda_n\beta^{-1}})\KL{\mu_{n-1,{\overline{B}_1}}}{\pi_{\overline{B}_1}}.\\
    \end{split}
\end{equation}
\end{proof}
An immediate consequence of Lemma~\ref{lemma:single_step_bcld} is that the KL-divergence is non-increasing, as stated in the following Corollary.
\begin{corollary}
    For all $i\in \{1,...,b\}$, $\KL{\mu_i}{\pi}\leq \KL{\mu_0}{\pi}$
\end{corollary}

\subsection{Proof of Lemma~\ref{lemma:cyclic_kl_general}}\label{sec:cyclic_proof}
\begin{proof}
    We prove the claim by induction on $b$. The claim is immediately evident for $b=1$ as a consequence of ~\ref{lemma:single_step_bcld} with $C_{\max}=C_1$, since $\overline{B}_1=\emptyset$. 

    Now we assume the inductive hypothesis for some $b-1\geq 1$ and prove the claim for $b\geq 2$ blocks.

    We start by applying Lemma~\ref{lemma:single_step_bcld} twice to obtain terms relating to step $b-2$, obtaining
 \begin{equation*}
        \begin{split}
            \KL{\mu_{b}}{\pi}\leq& C_b\KL{\mu_{b-1}}{\pi} + (1-C_b)\KL{\mu_{b-1, \overline{B}_b}}{\pi_{\overline{B}_b}} + D_b\\
            \text{Second descent expansion: }\leq& C_bC_{b-1}\KL{\mu_{b-2}}{\pi} + (1-C_b)\KL{\mu_{b-1, \overline{B}_b}}{\pi_{\overline{B}_b}} \\
            &+C_b(1-C_{b-1})\KL{\mu_{b-2, \overline{B}_{b-1}}}{\pi_{\overline{B}_{b-1}}} + D_b + C_{b}D_{b-1}\\.
            \end{split}
        \end{equation*}
        From here, we note that $\KL{\mu_{b-1, \overline{B}_b}}{\pi_{\overline{B}_b}}$ satisfies the theorem conditions, since all blocks in  $\overline{B}_b$ have been sampled. We can therefore apply the inductive hypothesis and obtain
 \begin{equation*}
        \begin{split}
            \KL{\mu_{b}}{\pi}\leq&C_bC_{b-1}\KL{\mu_{b-2}}{\pi} +
            C_{\max}(1-C_b)\KL{\mu_{0, \overline{B}_b}}{\pi_{\overline{B}_b}} + (1-C_b)\sum_{i=1}^{b-1}D_i \\
            &+C_b(1-C_{b-1})\KL{\mu_{b-2, \overline{B}_{b-1}}}{\pi_{\overline{B}_{b-1}}} + D_b + C_{b}D_{b-1}.\\
            \end{split}
        \end{equation*}
        Using Lemma ~\ref{lemma:kl_cond}, we can upper bound
        \[\KL{\mu_{b-2, \overline{B}_{b-1}}}{\pi_{\overline{B}_{b-1}}}\leq \KL{\mu_{b-2, \overline{B}_{b-1}|B_{b-1}}}{\pi_{\overline{B}_{b-1}|B_{b-1}}}\] 
        and then apply the chain lemma
        \[\KL{\mu_{b-2, \overline{B}_{b-1}|B_{b-1}}}{\pi_{\overline{B}_{b-1}|B_{b-1}}}=\KL{\mu_{b-2}}{\pi}-\KL{\mu_{b-2,B_{b-1}}}{\pi_{B_{b-1}}}\] to obtain
 \begin{equation*}
        \begin{split}
            \KL{\mu_{b}}{\pi}\leq&C_bC_{b-1}\KL{\mu_{b-2}}{\pi}\\
            & +C_b(1-C_{b-1})\KL{\mu_{b-2}}{\pi}-C_b(1-C_{b-1})\KL{\mu_{b-2,B_{b-1}}}{\pi_{B_{b-1}}} \\
            &+ (1-C_b)\sum_{i=1}^{b-1}D_i +C_{\max}(1-C_b)\KL{\mu_{0, \overline{B}_b}}{\pi_{\overline{B}_b}} + D_b + C_{b}D_{b-1}.\\
            \end{split}
        \end{equation*}
        We can define $\overline{B}_{b, b-1}\triangleq\overline{B}_b\cap\overline{B}_{b-1}$ (all variable blocks except the last two) and apply the chain lemma
        \[\KL{\mu_{0, \overline{B}_b}}{\pi_{\overline{B}_b}}=\KL{\mu_{0, \overline{B}_{b,b-1|B_{b-1}}}}{\pi_{\overline{B}_{b,b-1}|B_{b-1}}}+\KL{\mu_{0, B_{b-1}}}{\pi_{B_{b-1}}}.\]
        to obtain the bound
 \begin{equation*}
        \begin{split}
            \KL{\mu_{b}}{\pi}\leq&C_bC_{b-1}\KL{\mu_{b-2}}{\pi}\\
            & +C_b(1-C_{b-1})\KL{\mu_{b-2}}{\pi}-C_b(1-C_{b-1})\KL{\mu_{b-2,B_{b-1}}}{\pi_{B_{b-1}}} \\
            &+ (1-C_b)\sum_{i=1}^{b-1}D_i \\
            &+C_{\max}(1-C_b)\KL{\mu_{0, \overline{B}_{b,b-1|B_{b-1}}}}{\pi_{\overline{B}_{b,b-1}|B_{b-1}}} +C_{\max}(1-C_b)\KL{\mu_{0, B_{b-1}}}{\pi_{B_{b-1}}}\\
            &+ D_b + C_{b}D_{b-1}.\\
        \end{split}
    \end{equation*}
    We can regroup the terms and cancel $C_BD_{b-1}-C_BD_{b-1}=0$ yields
 \begin{equation*}
        \begin{split}
           \KL{\mu_{b}}{\pi}\leq&C_b\KL{\mu_{b-2}}{\pi}\\
            &-C_b(C_{\max}\KL{\mu_{0, \overline{B}_{b,b-1|B_{b-1}}}}{\pi_{\overline{B}_{b,b-1}|B_{b-1}}} + \sum_{i=1}^{b-2}D_i)\\
            & -C_b(1-C_{b-1})\KL{\mu_{b-2,B_{b-1}}}{\pi_{B_{b-1}}} \\
            &+C_{\max}\KL{\mu_{0, \overline{B}_{b,b-1|B_{b-1}}}}{\pi_{\overline{B}_{b,b-1}|B_{b-1}}} +C_{\max}(1-C_b)\KL{\mu_{0, B_{b-1}}}{\pi_{B_{b-1}}}\\
            &+ D_b + \sum_{i=1}^{b-1}D_i.\\
        \end{split}
    \end{equation*}
    By applying the inductive hypothesis in reverse, we can show
    \[-C_{b}(C_{\max}\KL{\mu_{0,\overline{B}_{b-1}|B_{b-1}}}{\pi_{\overline{B}_{b,b-1}|B_{b-1}}}+\sum_{i=1}^{b,b-1}D_i)\leq -C_{b}\KL{\mu_{b-1,\overline{B}_{b,b-1}|B_{b-1}}}{\pi_{\overline{B}_{b,b-1}|B_{b-1}}}.\]
    Substituting this into the second line, we have 
    \begin{equation}
        \begin{split}
           \KL{\mu_{b}}{\pi}\leq&C_b\KL{\mu_{b-2}}{\pi}\\
            &-C_b\KL{\mu_{b-2, \overline{B}_{b,b-1|B_{b-1}}}}{\pi_{\overline{B}_{b,b-1}|B_{b-1}}}-C_b(1-C_{b-1})\KL{\mu_{b-2,B_{b-1}}}{\pi_{B_{b-1}}} \\
            & +C_{\max}\KL{\mu_{0, \overline{B}_{b,b-1|B_{b-1}}}}{\pi_{\overline{B}_{b,b-1}|B_{b-1}}} +C_{\max}(1-C_b)\KL{\mu_{0, B_{b-1}}}{\pi_{B_{b-1}}}\\
            &+ \sum_{i=1}^bD_i.\\
        \end{split}
    \end{equation}

    We can once again expand the terms
    \begin{equation*}
        \begin{split}
            \KL{\mu_{b-2, \overline{B}_{b,b-1|B_{b-1}}}}{\pi_{\overline{B}_{b,b-1}|B_{b-1}}}\\
            C_{\max}\KL{\mu_{0, \overline{B}_{b,b-1|B_{b-1}}}}{\pi_{\overline{B}_{b,b-1}|B_{b-1}}}.
        \end{split}
    \end{equation*}
    Using the chain lemma and canceling the single block terms gives
    
    \begin{equation}
        \begin{split}
            \KL{\mu_{b}}{\pi}\leq&C_b\KL{\mu_{b-2}}{\pi}\\
            &-C_b\KL{\mu_{b-2,\overline{B}_b}}{\pi_{\overline{B}_b}}+C_bC_{b-1}\KL{\mu_{b-2,B_{b-1}}}{\pi_{B_{b-1}}} \\
            & +C_{\max}\KL{\mu_{0,\overline{B}_b}}{\pi_{\overline{B}_b}} -C_bC_{\max}\KL{\mu_{0, B_{b-1}}}{\pi_{B_{b-1}}}\\
            &+ \sum_{i=1}^bD_i.\\
        \end{split}
    \end{equation}
    Since $C_{\max}\geq C_{b-1}$ by definition and $\kl$ is non-increasing with respect to inner loop steps $b$, we can disregard \begin{equation}
    \begin{split}
        &C_bC_{b-1}\KL{\mu_{b-2,B_{b-1}}}{\pi_{B_{b-1}}}-C_bC_{\max}\KL{\mu_{0,B_{b-1}}}{\pi_{B_{b-1}}}\\
        &\leq C_bC_{b-1}\KL{\mu_{b-2,B_{b-1}}}{\pi_{B_{b-1}}}-C_bC_{\max}\KL{\mu_{b-2,B_{b-1}}}{\pi_{B_{b-1}}}
        \leq 0.
    \end{split}
    \end{equation}

    We now add zero to the right hand side via 
    \[0=C_b \KL{\mu_{b-2,B_b|\overline{B}_b}}{\pi_{B_b|\overline{B}_b}}-C_b\KL{\mu_{b-2,B_b|\overline{B}_b}}{\pi_{B_b|\overline{B}_b}}.\]
    We then have the three terms
    \begin{align}
        C_b\KL{\mu_{b-2}}{\pi} + \sum_{i=1}^bD_i\label{eqn:pos_mub},\\
        -C_b\KL{\mu_{b-2,\overline{B}_b}}{\pi_{\overline{B}_b}}-C_b\KL{\mu_{b-2,B_b|\overline{B}_b}}{\pi_{B_b|\overline{B}_b}}\label{eqn:neg_mub2_cl},\\
        C_{\max}\KL{\mu_{0,\overline{B}_b}}{\pi_{\overline{B}_b}}+C_b\KL{\mu_{b-2,B_b|\overline{B}_b}}{\pi_{B_b|\overline{B}_b}}\label{eqn:pos_mu0_cl}.
    \end{align}
    Note that the previous time steps left $\mu_{b-2,B_b|\overline{B}_b}$ invariant, hence $\KL{\mu_{b-2,B_b|\overline{B}_b}}{\pi_{B_b|\overline{B}_b}}=\KL{\mu_{0,B_b|\overline{B}_b}}{\pi_{B_b|\overline{B}_b}}$. Then by applying the chain lemma to ~\eqref{eqn:pos_mu0_cl} and ~\eqref{eqn:neg_mub2_cl}, we obtain
    \begin{equation}
        \begin{split}
            \KL{\mu_{b}}{\pi}\leq&C_b\KL{\mu_{b-2}}{\pi}-C_b\KL{\mu_{b-2}}{\pi}  +C_{\max}\KL{\mu_{0}}{\pi} +\sum_{i=1}^bD_i \\
            =&C_{\max}\KL{\mu_{0}}{\pi} +\sum_{i=1}^bD_i \\
        \end{split}
    \end{equation}
    which completes the proof.
\end{proof}

\subsection{CBLMC: Euler-Maruyama Discretization}
As discussed in the main text, Lemma ~\ref{lemma:cyclic_kl_general} can be used to trivially bound both the continuous ($C_i=e^{-2\gamma\lambda_i\beta^{-1}}$, $D_i=0$) and discrete ($C_i=e^{-\gamma\lambda_i\beta^{-1}}$, $D_i=3d_iL_i^2\lambda_i^2$) cases. As with RBLMC in the previous section, we take $\beta=1$ to simplify constant terms. Discrete-time CBLMC is shown in Algorithm~\ref{alg:cblmc}

    \begin{algorithm}
    \caption{Cyclic Block Langevin Monte Carlo (CBLMC)}
    \label{alg:cblmc}
        \begin{algorithmic}[1]
            \Procedure{CBLD}{$x_0\in\textrm{dom}(f)$, Block Permutation $\sigma=\{B_{1},...,B_b\}$, Step Sizes $\lambda\in\mathbb{R}^b_+$}
            \For{$k\geq 0$}
                \State Set $x^{k+1}_0=x^k$
                \For{$n=1$ to $b$}
                    \State Choose $i=\sigma_n$, sample $\xi_k\sim\mathcal{N}(0, I^d)$
            \begin{align}
                x^{k+1}_{n+1}&=x^{k+1}_{n}-\lambda_iU_i\nabla f(x^{k+1}_{n})+U_i\sqrt{2\lambda_i}\xi_k
            \end{align}
                \EndFor
                \State Set $x^{k+1}=x^{k+1}_{b+1}$
            \EndFor
            \EndProcedure
        \end{algorithmic}
    \end{algorithm}

We recall the convergence results for LMC from~\cite{vempala_rapid_2019} (adjusted using Lemma 16 of~\cite{chewi_analysis_2021} as in our analysis of RBLMC).
\begin{theorem}[LD Convergence (Theorem 2 of~\cite{vempala_rapid_2019})]
Let $\pi$ be a distribution satisfying an LSI with constant $\gamma$ with $L$-smooth potential. Assume that the LMC step size $\lambda$ is chosen such that $\lambda \in \left(0, \frac{\gamma}{4L^2}\right]$. Then after a single step of LMC, the distribution $\mu_{k+1}$ satisfies
\begin{equation}
    \KL{\mu_{k+1}}{\pi}\leq e^{-\gamma \lambda}\KL{\mu_{k}}{\pi} + 3 L^2d\lambda^2
\end{equation}
\label{thm:vempala_result}
\end{theorem}

For CBLMC, we additionally assume that the potential is $L$-smooth (Assumption~\ref{assmp:lipschitz}). From ~\cite{beck_convergence_2013}, this implies each block has a separate smoothness constant $L_i\leq L$. From applying Theorem~\ref{thm:vempala_result}, each block step has the descent

    \begin{equation}
        \begin{split}
            \KL{\mu^{kb}}{\pi}\leq&e^{-\gamma \lambda_i}\KL{\mu^{b-1}}{\pi} + (1-e^{-\gamma \lambda_i})\KL{\mu^{b-1}_{\overline{B_i}}}{\pi_{\overline{B_i}}} + 3L_i^2d_i\lambda_i^2. \\
        \end{split}
    \end{equation}
    When iterated for $kb$ cycles, we obtain the bound

    \begin{equation}
        \begin{split}
            \KL{\mu^{kb}}{\pi}\leq&e^{-\gamma kb\lambda_{\min}kb}\KL{\mu^{0}}{\pi} + \frac{4}{\gamma\lambda_{\min}}\sum_{i=1}L_i^2d_i\lambda_i^2. \\
        \end{split}
    \end{equation}
    where the constant terms are derived analogously to the RBLMC proof in the preceding section.

\section{Proof of Theorem ~\ref{theorem:w2_conv}}\label{appdx:stochastic}

We begin by recalling the following Lemmas from literature:
\begin{lemma}[Uniform $L^2$ bound on Langevin Diffusion (Lemma 3 of ~\cite{raginsky_non-convex_2017})]
\label{lemma:l2bound_stoch}
Let $f:\R^d\to\R$ be a differentiable function satisfying Assumption~\ref{assmp:dissapative}. For a random variable $x(t)=x(0)-\int_0^t\nabla f(x(s))ds + \int_0^t dW_s$, we have the bound
\begin{equation}
    \E[\Vert x(t)\Vert ^2]\leq \E[\Vert x(0)\Vert ^2]e^{-mt}+\frac{d/\beta+c}{m}(1-e^{-2mt}).
\end{equation}
\end{lemma}

\begin{lemma}[Wasserstein bound from Relative Entropy (Corollary 2.3 of ~\cite{bolley2005weighted})]
\label{lemma:weighted_tp_stoch}
    Let $\mu$, $\nu$ be two probability measures on some measurable space $X$ equipped with measurable distance $\mathscr{D}$, and let $\phi:X\to\R^+$ be a non-negative measurable function. Assume that $\exists x_0\in X$, $\alpha>0$ such that $\int e^{\alpha \mathscr{D}(x_0,x)^p}d\nu(x)$ is finite. Then
    \begin{equation}
        W_2\leq C\left[\KL{\mu}{\nu}^{1/2}+\left(\frac{\KL{\mu}{\nu}}{2}\right)^{1/4}\right]
    \end{equation}
    where
    \begin{equation}
        C\triangleq 2\inf_{x_0\in X}\left(\frac{1}{\alpha}(\frac{3}{2}+\log\int e^{\alpha \mathscr{D}(x_0,x)^p}d\nu(x))\right)^{1/p}.
    \end{equation}
\end{lemma}

In addition, we adapt the following Lemma from ~\cite{raginsky_non-convex_2017}
\begin{lemma}[Exponential $L^2$ Integrability of Block Langevin Diffusion]
\label{lemma:expint_stoch}
    Let $f:\R^d\to\R$ be a differentiable function satisfying Assumption~\ref{assmp:dissapative}, and let $x^k(t)=x(0)-\int_0^tU_k\nabla f(x(s))ds + \int_0^t U_kdW_s$ be a random variable in $\R^d$ across some number of iterations $k$, where $\sum_{i=1}^b U_i=I_d$. Suppose the initial state $x_0$ is drawn from some $\mu_0$ satisfying Assumption~\ref{assmp:init_density} and $\beta > 2/m$. Then on iteration $k$
    \begin{equation}
        \log E\left[e^{\Vert x^k_\lambda\Vert ^2}\right]\leq \kappa_0+2(c+\frac{d_{\max}}{\beta})k\lambda.
    \end{equation}
\end{lemma}

\begin{proof}
    Define $G(x^k_t)\triangleq e^{\Vert x^k_t\Vert ^2}$. By It\^{o}'s lemma, on iteration $k$ of $BLD$ we have
    \begin{equation}
    \begin{split}
        dG(x^k_t) =& -2\inner{x^k_t}{U_k\nabla f(x^k_t)}e^{\Vert x^k_t\Vert ^2}dt + \frac{2\beta^{-1}}{2}\mathrm{Tr}\left[U_k^2(2e^{\Vert x^k_t\Vert ^2}I+4xx^Te^{\Vert x^k_t\Vert ^2})\right]dt\\
        & + 2\sqrt{2\beta}\inner{x^k_t}{U_k}e^{\Vert x^k_t\Vert ^2}dW_t\\
        =&-2\inner{x^k_t}{U_k\nabla f(x^k_t)}G(t)dt \\&+ 2d_k\beta^{-1}G(x^k_t)dt + 4\Vert U_kx^k_t\Vert ^2\beta^{-1}G(x^k_t)dt + 2\sqrt{2\beta}\inner{x^k_t}{U_k}G(x^k_t)dW_t.
    \end{split}
    \end{equation}
    Integrating and summing across $k$ steps, we obtain
    \begin{equation}
        \begin{split}
            G(x^k_\lambda) =& G(x^0) + 2\sum_{i=1}^k \Biggl[\int_0^\lambda \left[-\inner{x^k_t}{U_k\nabla f(x^k_t)} + \beta^{-1}\Vert U_kx^k_t\Vert ^2)\right]G(x^k_t)dt \\&+ \int_0^\lambda  d_k\beta^{-1}G(x^k_t)dt + \int_0^\lambda\sqrt{2\beta}\inner{x^k_t}{U_k}G(x^k_t)dW_t\Biggr].
        \end{split}
    \end{equation}
    Applying the dissapativity condition and assuming $\beta > 2/m$, we can bound the first integrand as
    \begin{equation}
        \begin{split}
            -\inner{x^k_t}{U_k\nabla f(x^k_t)} + 2\beta^{-1}\Vert U_kx^k_t\Vert ^2\leq (2\beta^{-1}-m)\sum_{j\in B_i}(x^k_{t,j})^2 + c\leq c
        \end{split}
    \end{equation}
    which results in
    \begin{equation}
        G(x^k_\lambda) = G(x^0) + \sum_{i=1}^k 2(c +d_k\beta^{-1})\int^{\lambda}_0G(x^k_t)dt  +\int_0^\lambda2\sqrt{2\beta}G(x^k_t)\inner{x^k_t}{U_kdW_t}.
    \end{equation}
    As stated in~\cite{raginsky_non-convex_2017}, each It\^{o} integral $\int_0^\lambda2\sqrt{2\beta}G(x^k_t)\inner{x^k_t}{U_kdW_t}$ is a zero-mean Martingale. Taking expectations over both sides and applying Assumption~\ref{assmp:init_density} yields
    \begin{equation}
    \begin{split}
        \E[G(x^k_\lambda)] = \E[G(x^0)] + \sum_{i=1}^k 2(c +d_k\beta^{-1})\int^{\lambda}_0\E[G(x^k_t)]dt\\
        \leq e^{\kappa_0} + 2(c +d_{\max}\beta^{-1})\int^{k\lambda}_0\E[G(x^k_t)]dt.\\
    \end{split}
    \end{equation}
    where the integrability of $\E[G(x^k_t)]$ across block steps follows from the continuity of $x^k_t$ across each block step $k$. By Gr\"{o}nwall's inequality, we have
    \begin{equation}
        \begin{split}
            \E[G(x^k_\lambda)]\leq e^{\kappa_0}e^{2(c +d_{\max}\beta^{-1})k\lambda}\\
            \log\E[G(x^k_\lambda)]\leq \kappa_0 + 2(c +d_{\max}\beta^{-1})k\lambda\\
        \end{split}
    \end{equation}
\end{proof}

Theorem~\ref{theorem:w2_conv} follows as a consequence of Lemma~\ref{lemma:w2_dist} by applying the Otto-Villani theorem coupled with the triangle inequality for $W_2$ as stated in the main text.
\subsection{Proof of Lemma~\ref{lemma:w2_dist}}
\begin{proof}
Let $\mu_t^k$ and $\nu_t^k$ be the laws of SGBLD and BLD at times $t$ and iteration $k$ respectively with iterates $x^k(s)$, $y^k(s)$. We assume that each process selects the same variable blocks at each iteration, i.e. $B^k_x = B^k_{y}$.

Using the Girsanov formula, we can express the Radon-Nikodym derivative $\frac{d\nu_t^k}{d\mu_t^k}$ as
\begin{equation}
\begin{split}
    \frac{d\nu_t^k}{d\mu_t^k}=&\exp\biggl[\frac{\beta}{2}\int^t_0\inner{U_k\nabla f(y^k(s)) - U_kg_z(y^k(t))}{-U_k\nabla f(y^k(s))ds+\sqrt{2\beta^{-1}}U_kdW_s}\\&+\frac{\beta}{4}\int^t_0\inner{U_k\nabla f(y^k(s)) - U_kg_z(y^k(s))}{U_k\nabla f(y^k(s)) + U_kg_z(y^k(s))} dt\biggr]\\
    =&\exp\left[\sqrt{\frac{\beta}{2}}\int_0^t\inner{U_k\nabla f(y^k(s))-U_kg_z(y^k(s))}{dW_s}-\frac{\beta}{4}\int^t_0\Vert U_k\nabla f(y^k(s))-U_kg_z(y^k(s))\Vert ^2ds\right].
\end{split}
\end{equation}
Setting $t=\lambda_k$, we can express $\KL{\mu^k_t}{\nu^k_t}$ as
\begin{equation}
    \KL{\mu^k_t}{\nu^k}=-\int d\mu^k_t\log \frac{d\nu_t^k}{d\mu_t^k}=\sum_{i=1}^k\E\left[\frac{\beta}{4}\int^{\lambda}_0\Vert U_k\nabla f(y^k(s))-U_kg_z(y^k(s))\Vert ^2ds\right].
\end{equation}
since $\int_0^t\inner{U_k\nabla f(y^k(s))-U_kg_z(y^k(s))}{dW_s}$ is a 0-mean Martingale.

Using Assumption~\ref{assmp:grad_var} and Lemma~\ref{lemma:l2bound_stoch}, we obtain
\begin{equation}
\begin{split}
    \KL{\mu^k_t}{\nu^k_t}&=\sum_{i=1}^k\E\left[\frac{\beta}{4}\int^{\lambda}_0\Vert U_k\nabla f(y^i(s))-U_kg_z(y^i(s))\Vert ^2ds\right]\\
    &\leq\sum_{i=1}^k\left[\frac{\beta}{4}\int^{\lambda}_0M^2\E\Vert y^i(s)\Vert ^2+B^2ds\right]\\
    &\leq\sum_{i=1}^k\left[\frac{\beta}{4}\int^{\lambda}_{0}M^2(e^{-ms}\E\Vert y^i(0)\Vert ^2+\frac{d_i/\beta+c}{m}(1-e^{-ms}))+B^2ds\right].\\
\end{split}
\end{equation}
where we have applied Lemma~\ref{lemma:l2bound_stoch} in the last line. Integrating, we obtain
\begin{equation}
    \KL{\mu^k_t}{\nu^k_t}\leq\sum_{i=1}^k\frac{\beta}{4}\E\left[\frac{M^2}{m}(1-e^{-m\lambda})\E\Vert y^i(0)\Vert ^2+\frac{M^2(d_i/\beta+c)}{m^2}(mt+e^{-m\lambda}-1))+B^2\lambda\right].
\end{equation}

Expanding $e^{-m\lambda}$ and leveraging that  $m\lambda\geq 1-e^{-m\lambda}\geq m\lambda-\frac{m^2\lambda^2}{2}$
\begin{equation}
    \KL{\mu^k_\lambda}{\nu^k_\lambda}\leq\sum_{i=1}^k\frac{\beta M^2\lambda}{4}\E\Vert y^i(0)\Vert ^2+\frac{M^2\lambda^2(d_i+c\beta)}{4}+\frac{\beta B^2t}{4}.
\end{equation}

By repeatedly expanding Lemma~\ref{lemma:l2bound_stoch}, we obtain
\begin{equation}
\begin{split}
    \KL{\mu^k_t}{\nu^k_t}&\leq\sum_{i=0}^{k-1}\frac{M^2\beta\lambda}{4}\kappa_0+e^{-m(i-1)\lambda}\frac{M^2\lambda^2(d_i+\beta c)}{8}+\frac{M^2\lambda^2(d_i+\beta c)}{8}+\frac{\beta B^2\lambda k}{4}\\
    &\leq\frac{M^2\beta\lambda k}{4}\kappa_0+\frac{M^2\lambda^2(d_{\max}+\beta c)k}{4}+\frac{\beta B^2\lambda k}{4}\\
    &\triangleq (C_1 + C_2\lambda)\lambda k.
\end{split}
\end{equation}
where we have defined for convenience
\begin{equation}
    C_1\triangleq \frac{M^2\beta\kappa_0}{4} + \frac{\beta B^2}{4}
\end{equation}
and
\begin{equation}
    C_2\triangleq \frac{M^2( d_{\max}+\beta c)}{4}
\end{equation}

By Lemma~\ref{lemma:weighted_tp_stoch}, we can bound $W_2^2(\mu^k_t,\nu^k_t)$ as
    \begin{equation}
    \begin{split}
        W_2^2(\mu^k_t,\nu^k_t)\leq& 4C^2\left[\KL{\mu}{\nu}^{1/2}+\left(\frac{\KL{\mu}{\nu}}{2}\right)^{1/4}\right]^2.\\
    \end{split}
    \end{equation}

Setting $\alpha=1$, $d(x)=\Vert x\Vert ^{1/2}$, and $p=1/2$, we obtain from Lemma 11
\begin{equation}
    4C^2\leq (12+4\kappa_0+8(2c+\frac{d_{\max}}{\beta})k\lambda).
\end{equation}
    Note that for any $a\geq 0$, we have $(\sqrt{a}+(\frac{a}{2})^{1/4})^2\leq 2a+2\sqrt{a}$, since
    \begin{equation}
        \begin{split}
            (\sqrt{a}+(\frac{a}{2})^{1/4})^2=a+2^{3/4}a^{3/4}+\frac{a^{1/2}}{2^{1/2}}=a+(2^{1/4}a^{1/4})(2^{1/2}a^{1/2})+\frac{a^{1/2}}{2^{1/2}}.
        \end{split}
    \end{equation}
By Young's inequality, $(2^{1/4}a^{1/4})(2^{1/2}a^{1/2})\leq \frac{\sqrt{a}}{2^{1/2}}+a$, hence
    \begin{equation}
        \begin{split}
            (\sqrt{a}+(\frac{a}{2})^{1/4})^2=a+(2^{1/4}a^{1/4})(2^{1/2}a^{1/2})+\frac{a^{1/2}}{2^{1/2}}\leq 2a+\frac{2\sqrt{a}}{\sqrt{2}}\leq 2a+2\sqrt{a}.
        \end{split}
    \end{equation}

plugging in Lemma~\ref{lemma:expint_stoch}, and assuming $k\lambda \geq 1$, $k> \lambda$ we have
\begin{equation}\label{eqn:w2_stoch_bound}
    \begin{split}
    W_2^2(\mu^k_t&,\nu^k_t)\leq 2C^2\left[\KL{\mu}{\nu}+\sqrt{\KL{\mu}{\nu}}\right]
    \end{split}
\end{equation}
\begin{equation}
    \begin{split}
    \leq& (12+8(\kappa_0+(2c+d_{\max}/\beta)))\biggl[(C_1 + C_2\lambda)k\lambda +\sqrt{(C_1 + C_2\lambda)}k\lambda\biggr](k\lambda)\\
    \leq& (12+8(\kappa_0+(2c+d_{\max}/\beta)))\biggl[( C_2+\sqrt{C_2})\sqrt{\lambda k} +(C_1 + \sqrt{C_1)}\biggr](k\lambda)^2\\
    =& C_0^2\biggl[( C_2+\sqrt{C_2})\sqrt{\lambda k} +(C_1 + \sqrt{C_1)}\biggr](k\lambda)^2.\\
    \end{split}
\end{equation}
We thereby obtaining Lemma~\ref{lemma:w2_dist} with:
\begin{equation}
    \begin{split}
        C_0&\triangleq (12+8(\kappa_0+(2c+d_{\max}/\beta))),\\
        C_1&\triangleq \frac{M^2\beta\kappa_0}{4} + \frac{\beta B^2}{4},\\
    C_2&\triangleq \frac{M^2( d_{\max}+\beta c)}{4}.\\
    \end{split}
\end{equation}

\end{proof}
\subsection{Constants in Expected Function Gap Bounds}
We start by recalling the Lemma from~\cite{polyanskiy_wasserstein_2016}:
\begin{lemma}[Wasserstein Continuity for Quadratic-Growth Potentials]
Let $\mu$, $\pi$ be probability distributions with finite second moments and let $f:\R^d\to\R^+$ be a continuously differentiable function satisfying $\Vert \nabla f(x)\Vert ^2\leq c_1\Vert x\Vert ^2 + c_2$. Then we have
\begin{equation}
    \left |\int f(x)d\mu(x)-\int f(x)d\pi(x)\right|\leq (c_1\sigma+c_2)W_2(\mu, \pi) 
\end{equation}
where $\sigma=\sqrt{\max[\E_\mu[\Vert x\Vert ^2],\; \E_{\pi}[\Vert x\Vert ^2]]}$.

\end{lemma}
~\cite{raginsky_non-convex_2017} bound the constant $\sigma^2=\max{\E_{\mu^k}[x^2], \E_{\pi}[x^2]}$ using an unbiased oracle. As discussed in the main text, DXs have fixed device variation from analog errors, precluding unbiased estimation. However, DX errors take the form of perturbations in the underlying function, i.e. the target function characteristics are intact. For instance, DXs with quadratic potential targets~\citep{aifer_thermodynamic_2023,song_ds-gl_2024} are still optmizing/sampling quadratic functions. Accordingly, Assumptions~\ref{assmp:lipschitz} and ~\ref{assmp:dissapative}, that the DX gradient retains both Lipschitz continuity and dissipativity, are reasonable. Assuming $g_\delta$ is $(\mathfrak{m},\mathfrak{c})$-dissipative, we have from Lemma 3 of ~\cite{raginsky_non-convex_2017}:
\begin{equation}
    \Vert x(t)\Vert ^2\leq \kappa_0+\frac{\mathfrak{c}+d/\beta}{\mathfrak{m}}.
\end{equation}
Then
\begin{equation}
    \sigma^2=\kappa_0+\max\left[ \frac{c+d/\beta}{m}, \frac{\mathfrak{c}+d/\beta}{\mathfrak{m}}\right].
\end{equation}

\end{document}